\documentclass[sigconf]{acmart}

\usepackage{forest}
\usepackage{float}
\usepackage{stfloats}
\usepackage{booktabs}
\usepackage{multirow}
\usepackage{threeparttable}
\usepackage{siunitx}
\usepackage{tabularx}
\usepackage[table]{xcolor}
\usepackage{enumitem}
\usepackage{amssymb}

\AtBeginDocument{
  }

\copyrightyear{2026}
\acmYear{2026}
\setcopyright{cc}
\setcctype{by}
\acmConference[KDD '26]{Proceedings of the 32nd ACM SIGKDD Conference on Knowledge Discovery and Data Mining V.2}{August 09--13, 2026}{Jeju Island, Republic of Korea}
\acmBooktitle{Proceedings of the 32nd ACM SIGKDD Conference on Knowledge Discovery and Data Mining V.2 (KDD '26), August 09--13, 2026, Jeju Island, Republic of Korea}
\acmDOI{10.1145/3770855.3818998}
\acmISBN{979-8-4007-2259-2/2026/08}

\begin{document}

\title[Co4ICF]{Co4ICF: Co-evolving Physics-Informed Surrogate and RL-based Pulse Optimizer for Inertial Confinement Fusion}

\author{Jiatong Zhao}
\authornote{These authors contributed equally to this research.}
\orcid{0009-0004-2327-2604}
\email{zhaojiatong@sjtu.edu.cn}
\affiliation{
  \department{School of Physics and Astronomy}
  \department{Zhiyuan College}
  \institution{Shanghai Jiao Tong University}
  \city{Shanghai}
  \country{China}
}

\author{Tengyue Zhang}
\authornotemark[1]
\orcid{0009-0005-4752-7584}
\email{zhangty_23@sjtu.edu.cn}
\affiliation{
  \department{School of Physics and Astronomy}
  \department{Zhiyuan College}
  \institution{Shanghai Jiao Tong University}
  \city{Shanghai}
  \country{China}
}

\author{Yuhan Wang}
\authornotemark[1]
\orcid{0009-0002-1171-2321}
\email{wang.yuhan@sjtu.edu.cn}
\affiliation{
  \department{School of Physics and Astronomy}
  \department{Zhiyuan College}
  \institution{Shanghai Jiao Tong University}
  \city{Shanghai}
  \country{China}
}

\author{Fuyuan Wu}
\email{fuyuan.wu@sjtu.edu.cn}
\affiliation{
  \department{School of Physics and Astronomy}
  \institution{Shanghai Jiao Tong University}
  \city{Shanghai}
  \country{China}
}

\author{Junchi Yan}
\authornote{Corresponding author.}
\orcid{0000-0001-9639-7679}
\email{yanjunchi@sjtu.edu.cn}
\affiliation{
  \department{School of Artificial Intelligence}
  \institution{Shanghai Jiao Tong University}
  \city{Shanghai}
  \country{China}
}

\renewcommand{\shortauthors}{Jiatong Zhao, Tengyue Zhang, Yuhan Wang, Fuyuan Wu and Junchi Yan}

\begin{abstract}
Offline-trained surrogates for Inertial Confinement Fusion (ICF) suffer a well-known failure mode that iterative optimizers drive inputs into out-of-distribution (OOD) regions where predictions become unreliable.
Here we present \textbf{Co4ICF}, a co-evolving framework that couples a physics-informed surrogate with a PPO-based pulse optimizer.
The surrogate is iteratively fine-tuned on policy-induced trajectories, correcting extrapolation errors as the optimizer shifts the input distribution; the optimizer queries this evolving surrogate as a fast environment.
In the 1D MULTI environment, Co4ICF achieves \textbf{146.1\%} normalized yield based on current laser design baseline; as a post-hoc cross-fidelity check, the optimized pulse further attains \textbf{246.9\%} normalized yield when directly evaluated in 2D-MULTI without any 2D training or fine-tuning.
Budget-matched ablations support that the gains are not explained solely by additional simulation data and are consistent with the co-evolving mechanism playing a key role.
We release a large-scale MULTI-IFE simulation dataset to support future benchmarking.
\end{abstract}

\ccsdesc[500]{Applied computing~Physics}
\ccsdesc[500]{Computing methodologies~Modeling methodologies}
\ccsdesc[300]{Computing methodologies~Machine learning}
\ccsdesc[300]{Computing methodologies~Artificial intelligence}

\keywords{Inertial Confinement Fusion, Surrogate Modeling, Reinforcement Learning, Pulse Optimization}

\begin{teaserfigure}
  \centering
  \includegraphics[width=0.9\textwidth]{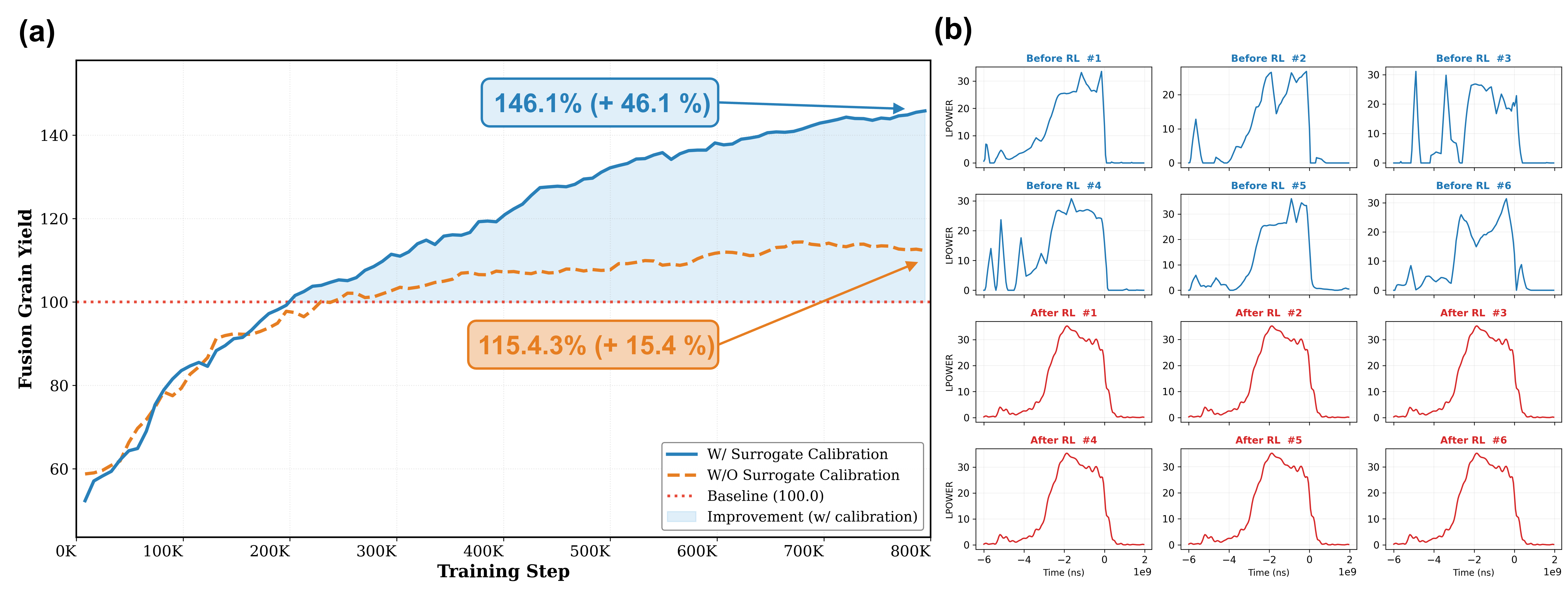}
  \caption{\textbf{Performance of Co4ICF in laser pulse optimization.}
  \textbf{(a) Fusion Yield:} The Co4ICF framework (blue solid line) dynamically calibrates the surrogate, reaching \textbf{146.1\%} normalized yield in the 1D-MULTI optimization loop, i.e., a \textbf{46.1\%} improvement over the designed baseline (red dotted line). This outperforms the static surrogate approach (orange dashed line, +15.4\%), which suffers from distribution shift. (Re-evaluated by 1D-MULTI; see Table~\ref{tab:opt_comparison} for final direct 2D-MULTI evaluation.)
  \textbf{(b) Pulse Optimization:} Comparison between the initial random pulse samples (top rows) and the final RL-optimized waveforms (bottom rows), showing the pulse structures found by the policy.
  }
  \Description{Left panel shows fusion grain yield over training steps, with the calibrated surrogate curve rising above the baseline while the uncalibrated curve plateaus lower. Right panels compare noisy before-RL pulse waveforms with smoother after-RL optimized waveforms.}
  \label{fig:teaser_ppo}
\end{teaserfigure}


\maketitle

\begin{figure*}[!b]
  \centering
  \includegraphics[width=0.8\textwidth]{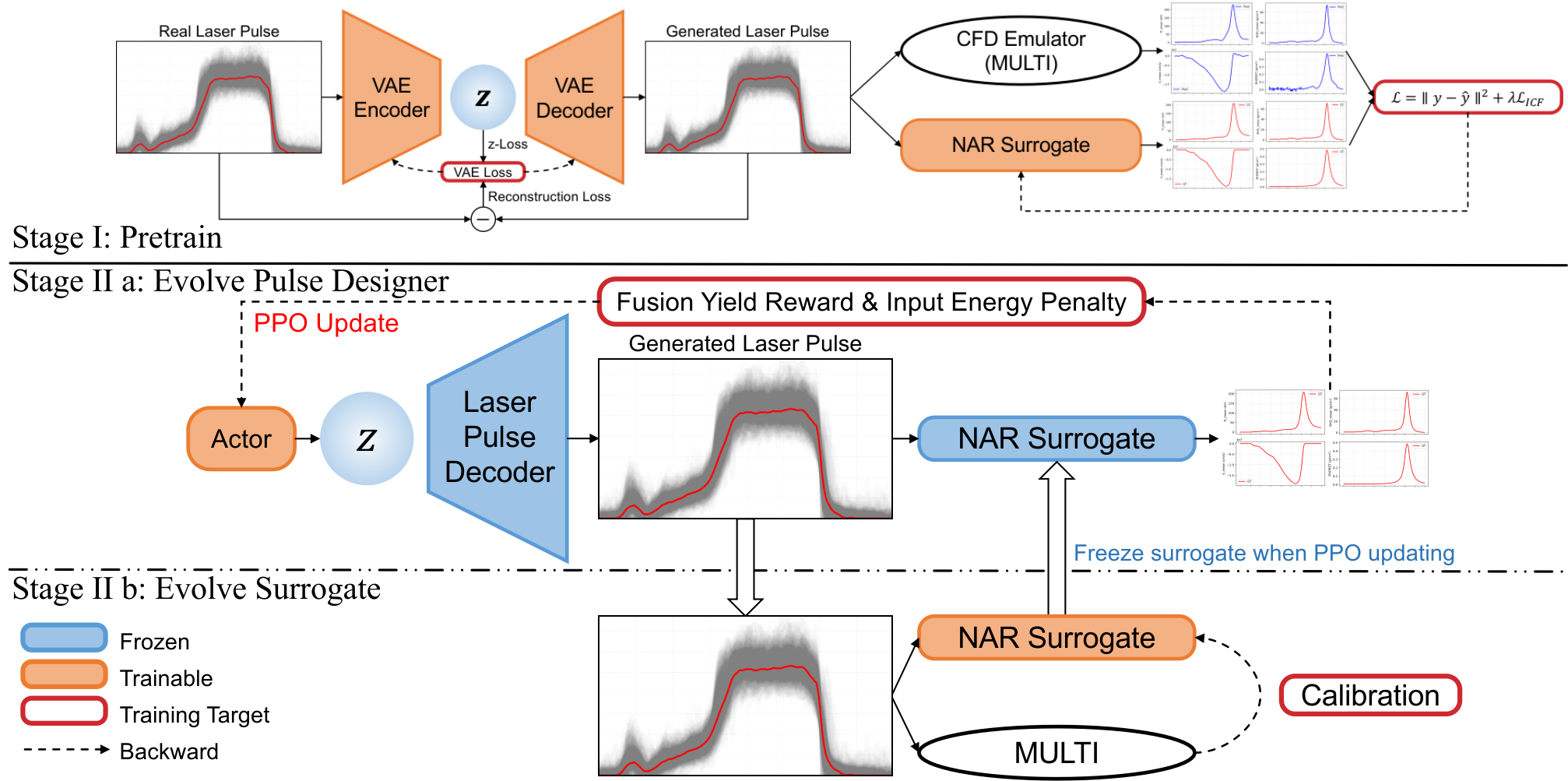}
  \caption{
      Co4ICF framework.
      \textbf{Stage~I (Pretrain)}: A VAE decoder learns a latent representation of laser pulses; the surrogate is trained to mimic MULTI's dynamics with physics-informed regularization.
      \textbf{Stage~II (Evolve)}: The actor samples latent codes, decodes them into pulses, and queries the NAR surrogate for rewards. High-frequency PPO updates optimize the actor; low-frequency surrogate fine-tuning on MULTI-verified on-policy samples corrects distribution shift.
  }
  \Description{Framework diagram showing offline VAE pretraining, pulse decoding, surrogate prediction, PPO actor updates, and periodic surrogate refinement with MULTI-labeled on-policy pulses.}
  \label{fig:teaser}
\end{figure*}

\section{Introduction}
\label{sec:introduction}

High-fidelity numerical simulations are a cornerstone of Inertial Confinement Fusion (ICF) research, but expensive simulations and manual parameter tuning limit their use in design optimization.
Recent AI for Fusion research has begun to alleviate these bottlenecks by applying offline pretrained deep neural surrogates \cite{humbird2019machine,ejaz2024deep,wang2025predictive} and RL-based optimization policies \cite{degrave2022magnetic,capuano2025shaping}.
Given the high dimensional design space, it is difficult for an offline-trained surrogate to span a broad solution manifold, so the performance can degrade sharply under OOD inputs during iterative optimization \cite{trabucco2022design}.
RL methods in contrast require massive on-policy interaction for policy updates, while using high-fidelity simulators as the training environment is not computationally practical.
A further concern is physical consistency, that surrogates and policies trained purely on data can violate conservation laws or known scaling relations \cite{raissi2019physics,karniadakis2021physics}.

To address these bottlenecks, we present \textbf{Co4ICF}: a co-evolving ICF design framework that couples a physics-informed surrogate with a PPO pulse optimizer.
As illustrated in Fig.~\ref{fig:teaser}, Co4ICF couples two modules:
(i) a physics-informed \emph{surrogate} that replaces expensive radiation hydrodynamics simulations (MULTI), and
(ii) an RL-based laser pulse \emph{designer} that parameterizes and searches laser pulses under feasibility and energy constraints.
The core mechanism is a \emph{co-evolving loop}: the designer is updated at high frequency using the surrogate as a cheap environment, while the surrogate is periodically fine-tuned at lower frequency on policy-induced trajectories to correct extrapolation errors caused by distribution shift.

The insight behind this loop is that optimizer-induced OOD drift, conventionally treated as a failure mode, can be recast as a training signal: the very distribution shift that degrades a static surrogate instead provides on-policy data for its refinement.
In the 1D-MULTI optimization loop, Co4ICF reaches \textbf{146.1\%} normalized yield relative to the designed-pulse baseline while delivering a \textbf{990$\times$} rollout-time reduction for high-frequency policy updates.
As a post-hoc cross-fidelity check, the final optimized pulse is directly evaluated in 2D-MULTI and reaches \textbf{246.9\%} normalized yield under the same baseline normalization, without using any 2D samples for training or fine-tuning.
Budget-matched ablations confirm that these gains are not explained solely by additional simulation data and are consistent with the co-evolving mechanism playing a key role.
We provide full access to Co4ICF dataset and implementation: dataset at: \url{https://huggingface.co/datasets/Oyhs/Co4ICFDataset}; code at: \url{https://github.com/Co4ICF/co4icf}.

In short, our main contributions are:
\begin{enumerate}[left=0em]
    \item We propose a \emph{co-evolving} framework that jointly updates an ICF surrogate and a pulse designer, addressing optimizer induced distribution shift in iterative learned-model-based loops.
    \item We develop a \emph{PPO-based} pulse optimizer that searches in a low-dimensional VAE latent pulse manifold and queries the physics-regularized surrogate as a fast environment under energy and feasibility constraints.
    \item Our pipeline reaches \textbf{146.1\%} normalized yield in the 1D-MULTI optimization loop and \textbf{246.9\%} normalized yield in post-hoc direct 2D-MULTI evaluation, while accelerating high frequency policy rollouts by \textbf{$\sim$990$\times$} within 1D search space.
    \item We release a large-scale 1D MULTI simulation dataset and the implementation to facilitate reproducible benchmarking.
\end{enumerate}

\section{Background}
\label{sec:background}

\subsection{Inertial Confinement Fusion and Simulation}
\label{sec:background:icf}

Inertial Confinement Fusion (ICF) is one of the major approaches toward realizing controlled nuclear fusion energy. ICF uses high-energy drivers (like lasers or pulsed power) to rapidly compress a fuel capsule to extreme densities and temperatures, triggering a thermonuclear burn wave before the target disassembles.
In December 2022, the National Ignition Facility (NIF) achieved target gain $G>1$ for the first time, producing 3.15 MJ of fusion yield from 2.05 MJ of laser energy, and subsequent experiments \cite{hurricane2024ignition} have replicated this result.
\begin{figure}[H]
    \centering
    \includegraphics[width=0.4\linewidth]{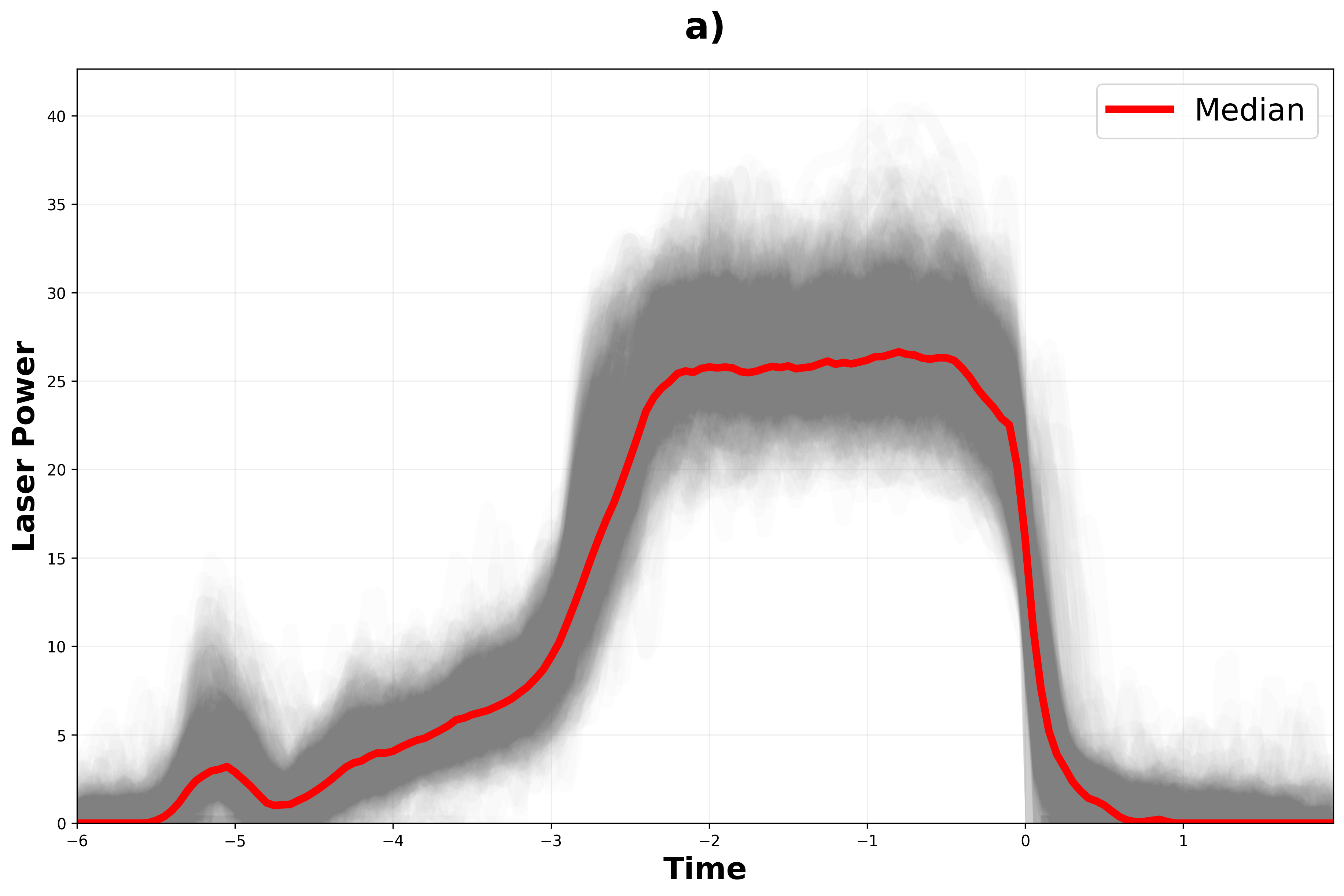}
    \includegraphics[width=0.4\linewidth]{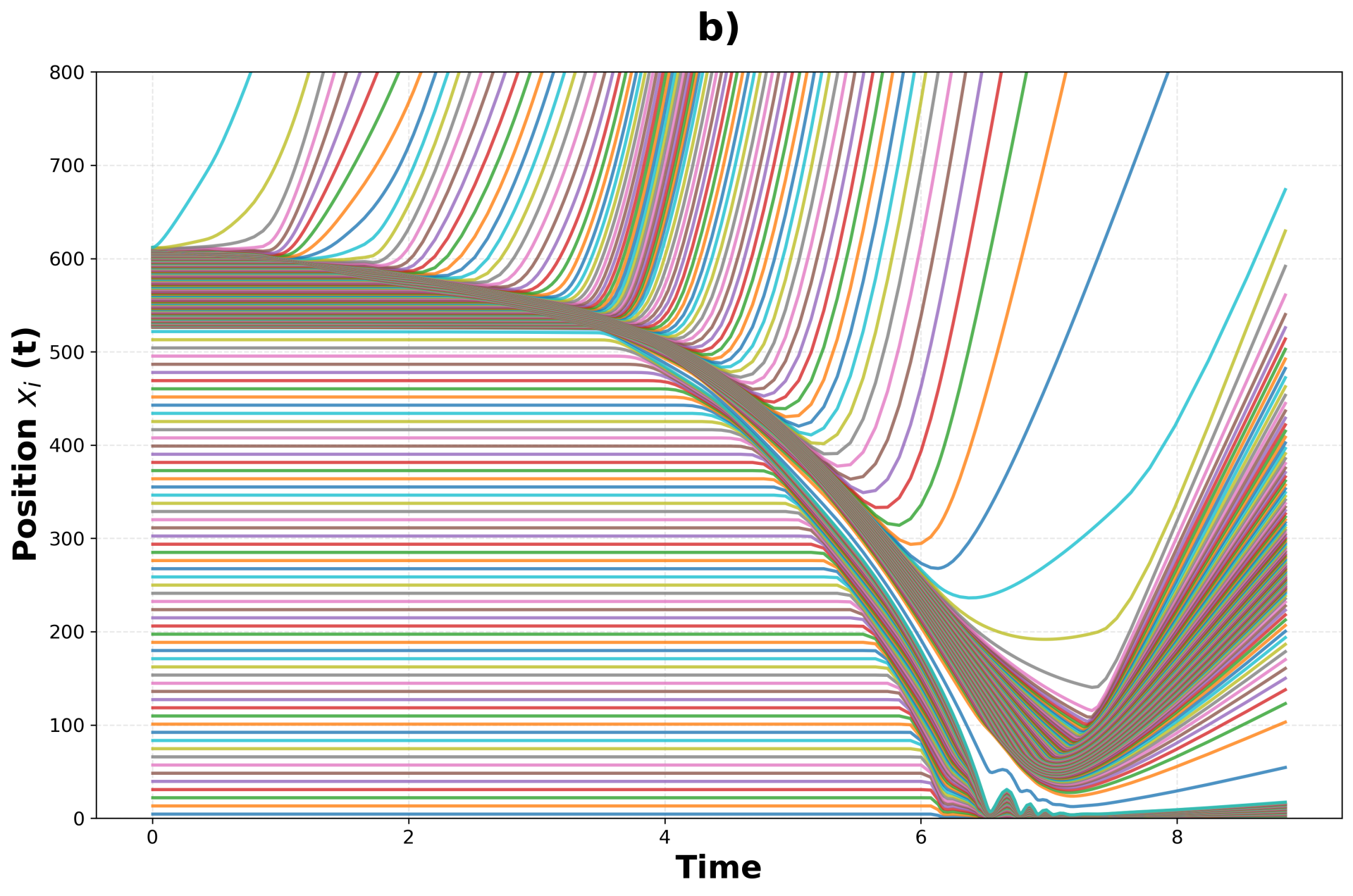}
    \caption{Inputs and outputs of the MULTI-IFE simulation. \textbf{(a)} Distribution of input laser pulses across the dataset. \textbf{(b)} an example 1D-MULTI output: each curve corresponds to a Lagrangian mass layer, showing ablation-driven compression (decreasing radius) followed by ignition.}
    \Description{Two-panel simulation overview: the left plot overlays many laser power pulses with a red median curve, and the right plot shows Lagrangian layer positions compressing inward and then rebounding over time.}
    \label{fig:MULTI-results}
\end{figure}

Given the extreme nonlinearities and prohibitive experimental costs, large-scale numerical simulation is indispensable for interpreting implosion dynamics and optimizing pulse designs.
This work uses MULTI-IFE (also ``MULTI''), an open-source radiation hydrodynamics code \cite{ramis2016multi}. 
MULTI is a Lagrangian code designed for Inertial Fusion Energy studies; it captures multigroup radiation diffusion, separate electron/ion temperatures, and tabulated equations of state while remaining computationally efficient. The three-stage implosion dynamics captured by MULTI are detailed in Appendix~\ref{sec:appendix:icf_stages}.

\subsection{Related Work}
\label{sec:background:surrogate}

\paragraph{Physics-informed surrogates.}
Such constraints improve data efficiency and generalization in PDE surrogates. 
Physics-informed Neural Networks (PINNs) \cite{raissi2019physics} enforce governing equations as soft penalties, and this idea has been adapted to ICF surrogate modeling. 
Early ICF surrogates \cite{humbird2019machine,anirudh2020improved} focused on efficient mapping and physical consistency.
Recent work uses facility-specific architectures \cite{ejaz2024deep} and Transformers \cite{olson2024transformer,wang2025predictive}. 
Transfer learning \cite{humbird2021cognitive} bridges the simulation--experiment gap with sparse experimental data, culminating in predictive capabilities for NIF ignition \cite{spears2025predicting}. 
However, a shared limitation is that these surrogates are trained offline and can degrade sharply under OOD inputs, which is precisely the regime that iterative optimizers actively explore.

\paragraph{ICF pulse optimization.}
On the optimization side, early ICF design relied on heuristic and evolutionary strategies \cite{tao2023laser,li2023hybrid}, later augmented with genetic algorithms \cite{wu2022machine} and DNN surrogates \cite{wei2024machine}. 
Bayesian Optimization (BO) has since become a standard approach, from multi-fidelity simulation frameworks \cite{wang2024multifidelity} to automated experimental campaigns \cite{gopalaswamy2025automated}. 
RL remains scarce in ICF compared to Tokamak control \cite{degrave2022magnetic}, with Capuano et al.\ \cite{capuano2025shaping} as a rare application for pulse shaping. 
A key difficulty common to all iterative methods is that training on a static, offline surrogate leads to \emph{model exploitation}: the optimizer discovers OOD inputs that score high under the surrogate but are physically invalid, motivating approaches where the surrogate and optimizer adapt jointly.

\paragraph{Evolving surrogate and optimizer.}
The idea of alternating model refinement with optimization is well established: the Dyna architecture in model-based RL interleaves real experience with model updates \cite{sutton1991dyna}, Surrogate-Assisted Evolutionary Algorithms periodically retrain the surrogate on evaluated candidates \cite{liu2024saea}, and offline Model-Based Optimization methods address distribution shift through conservative objectives or data-manifold constraints \cite{trabucco2022design,kim2024offline}. A common thread is that these methods either penalize predictive uncertainty or restrict search to the offline data manifold to avoid OOD exploitation.

\section{Co4ICF}
\label{sec:method}

We present \textbf{Co4ICF}, a closed-loop framework that jointly optimizes a pulse designer and a physics-informed surrogate.

\subsection{Pipeline}
\label{sec:method:pipeline}

\begin{figure}[htbp]
    \centering
    \includegraphics[width=\linewidth]{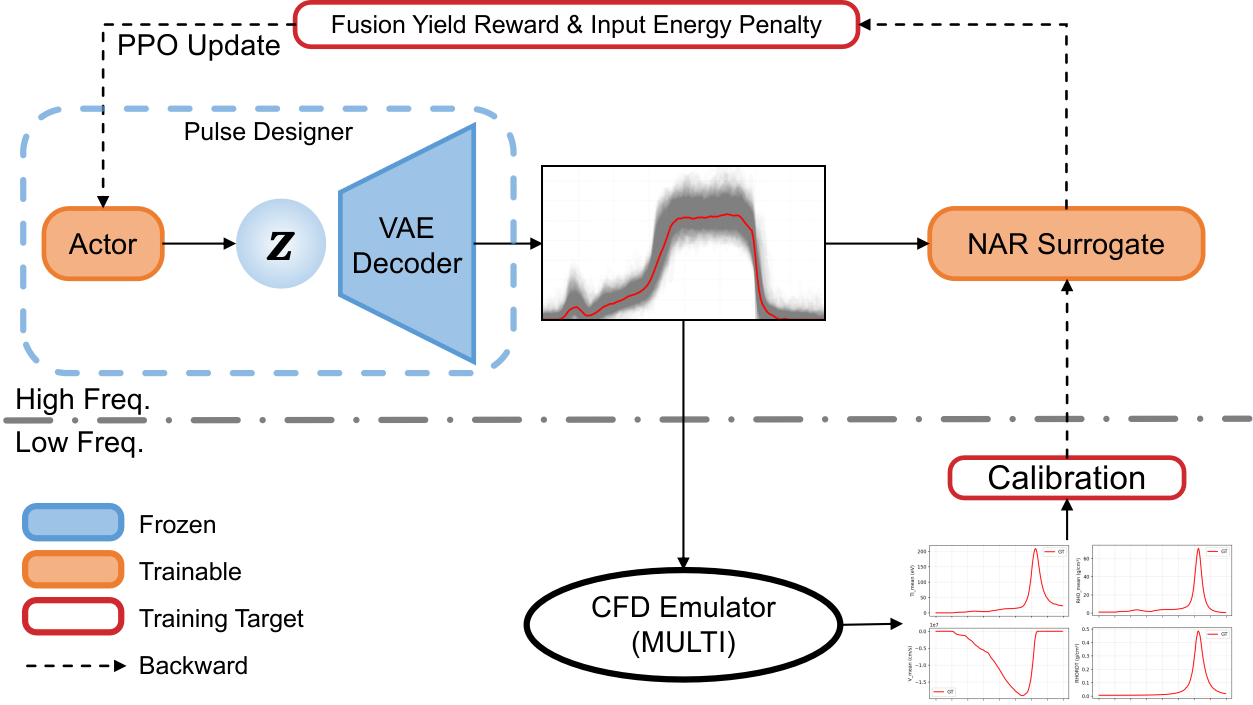}
    \caption{Co-evolving Stage of Co4ICF. The PPO actor queries the surrogate at high frequency for policy updates; every $K$ steps, on-policy pulses are evaluated with MULTI and the surrogate is fine-tuned on the augmented dataset.}
    \Description{Co-evolving optimization diagram in which an actor samples a latent code, a VAE decoder produces a laser pulse, a NAR surrogate evaluates the pulse, and periodic simulator feedback updates the surrogate.}
    \label{fig:coevolve}
\end{figure}

\paragraph{Overview.}
Co4ICF is a closed-loop framework with a \emph{pulse designer} and a \emph{physics-informed surrogate} (Fig.~\ref{fig:coevolve}).
Let $\mathbf{t}\in\mathbb{R}^{d_t}$ denote static target/capsule parameters and $\mathbf{p}\in\mathbb{R}^{T}$ denote a laser power waveform discretized on a fixed grid.
The \emph{designer} operates in a low-dimensional latent action space: an \emph{actor}, or policy network $\pi_{\omega}(\cdot\mid \mathbf{t})$ outputs a distribution over latent codes $\mathbf{z}\in\mathbb{R}^{d_z}$, and a pretrained \emph{decoder} $\mathbf{D}_{\psi}$ maps $\mathbf{z}$ to a feasible waveform $\mathbf{p}=\mathbf{D}_{\psi}(\mathbf{z})$.
Given $(\mathbf{t},\mathbf{p})$, the \emph{surrogate} $S_{\theta}$ predicts time-resolved trajectories $\hat{\mathbf{y}}=S_{\theta}(\mathbf{t},\mathbf{p})$ from which we compute a scalar optimization objective $r=R(\hat{\mathbf{y}})$ that encodes ignition reward together with feasibility/energy penalties.
The designer proposes candidate pulses and the surrogate provides fast, physically consistent evaluations. To mitigate OOD drift and reward hacking, we interleave \emph{high-frequency} actor updates with \emph{low-frequency} surrogate refinement on newly queried on-policy trajectories.

\paragraph{Stage~I: offline pretraining.}
We use offline MULTI data $\mathcal{D}_0$.
Each sample contains target parameters, a laser pulse, and a MULTI trajectory: $(\mathbf{t}_i,\mathbf{p}_i,\mathbf{y}_i)$ (Section~\ref{sec:method:dataset}).
First, we pretrain the VAE on a broad pulse subset to learn a feasible pulse manifold, and retain its decoder $\mathbf{D}_{\psi}$ as a fixed mapping from the latent prior to waveform space. This converts waveform search into a lower-dimensional control problem and constrains exploration to the learned manifold.
In parallel, we pretrain the surrogate $S_{\theta}$ to regress MULTI trajectories.
Beyond standard supervised regression, we introduce inequality-based physics regularization to discourage non-physical trajectories and improve extrapolation stability (Section~\ref{sec:method:physics_loss}).

\paragraph{Stage~II: co-evolving optimization with dual-frequency updates.}
After pretraining, we alternate between fast policy improvement against the surrogate and periodic surrogate refinement on on-policy samples.
In the high-frequency loop, we treat $S_{\theta}$ as an inexpensive environment: at each PPO iteration, the actor samples $\mathbf{z}\sim\pi_{\omega}(\cdot\mid\mathbf{t})$, decodes it into $\mathbf{p}=\mathbf{D}_{\psi}(\mathbf{z})$, queries $S_{\theta}$ for $\hat{\mathbf{y}}=S_{\theta}(\mathbf{t},\mathbf{p})$, and computes the reward $r=R(\hat{\mathbf{y}})$; PPO then updates the policy using batches of surrogate-generated rollouts.
In the low-frequency loop, we refine $S_{\theta}$ every $K$ PPO updates: we sample pulses from the current policy, re-evaluate a budgeted subset with MULTI, and fine-tune the surrogate on the union of the offline data and these newly labeled on-policy samples.

\subsection{Dataset}
\label{sec:method:dataset}

\begin{figure}[!htbp]
    \centering
      \begin{forest}
      for tree={
          grow'=east,
          parent anchor=east,
          child anchor=west,
          align=left,
          edge={-latex},
          rounded corners,
          draw,
          fill=gray!8,
          font=\small,
          l sep=10pt,
          s sep=6pt,
          inner xsep=6pt,
          inner ysep=4pt,
      }
      [Dataset
          [Real (701)
          [Measured input pulses]
          ]
          [Clean (10k)
          [VAE augmented samples]
          ]
          [Other (80k)
          [Fourier (40k): freq. noise]
          [PPCA (40k): kernel samples]
          ]
      ]
      \end{forest}
    \caption{Dataset composition. \emph{Real}: 701 measured pulses; \emph{Clean}: 10k VAE-augmented in-distribution samples; \emph{Other}: 80k OOD pulses from Fourier noise and PPCA sampling.}
    \Description{Tree diagram of the dataset split into Real measured pulses, Clean VAE-augmented samples, and Other out-of-distribution Fourier and PPCA samples.}
    \label{fig:dataset_tree}
\end{figure}
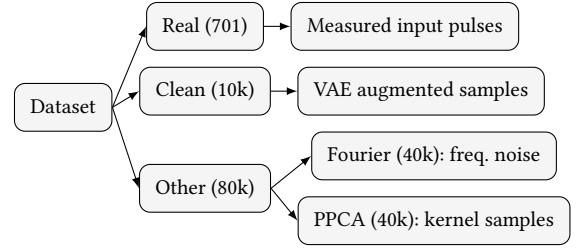

A major challenge in AI for fusion is the lack of a common benchmark.
We construct a synthesized simulation dataset that combines multiple input sampling strategies to broaden coverage while preserving physical plausibility.
The dataset contains three subsets (Fig.~\ref{fig:dataset_tree}).
The \emph{Real} subset contains measured laser pulses from ICF experiments, i.e., partially optimized solutions in a limited design space.
The \emph{Clean} subset samples a VAE trained only on \emph{Real}, yielding \emph{in-distribution} pulses on the learned manifold; this dataset VAE is distinct from the optimization VAE.
The \emph{Other} subset uses noise augmentation to cover \emph{out-of-distribution} patterns.
Each pulse is simulated with MULTI to obtain time-resolved implosion trajectories and associated physical quantities.
Details are in Appendix~\ref{sec:appendix:data}.

\subsection{Surrogate with Physics-Informed Regularization}
\label{sec:method:surrogate}
\label{sec:method:physics_loss}

The surrogate maps laser pulse inputs to time-resolved implosion trajectories, replacing slow MULTI calls with fast forward evaluation.
The input is a 160-step laser history $\mathbf{p}$ plus three static capsule parameters $\mathbf{t}$.
We concatenate them into $\mathbf{x}=\mathrm{Concat}(\mathbf{t},\mathbf{p})$.
The output contains five 160-step trajectories: fuel density, velocity, temperature, ablation-boundary index, and boundary radius.

\begin{figure}[htbp]
    \centering
    \includegraphics[width=\linewidth]{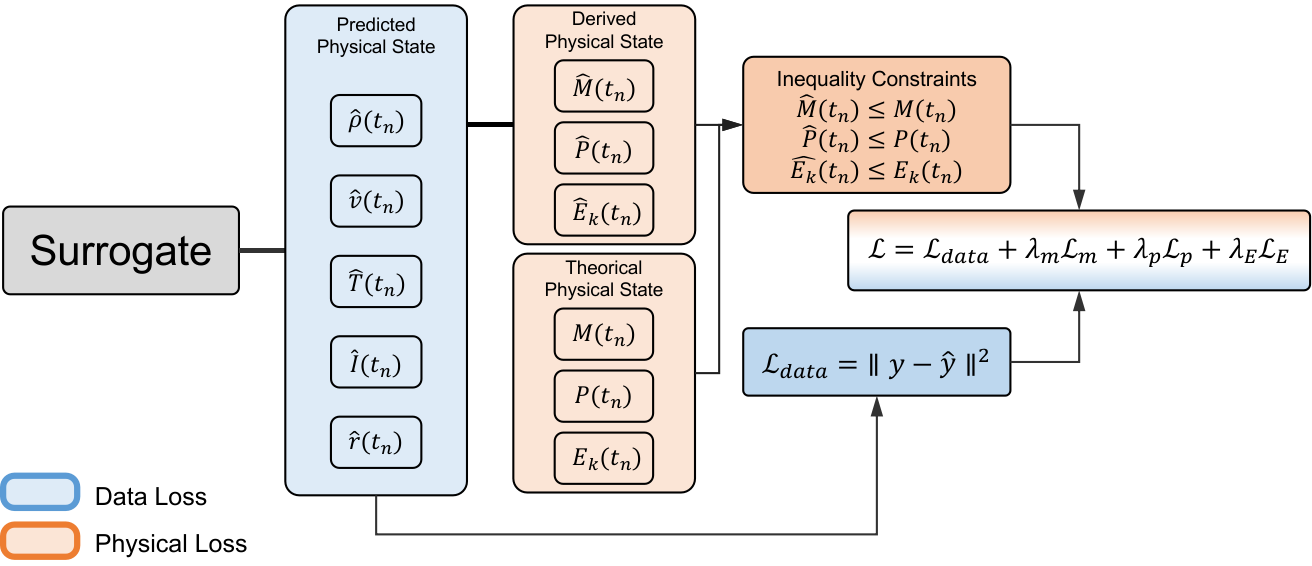}
    \caption{Physics-Informed Loss. The surrogate predicts trajectories from which boundary index, radius, and velocity are extracted. Steady-ablation budgets then yield inequality constraints on mass, momentum, and energy via squared hinge penalties.}
    \Description{Flow diagram of the physics-informed loss layer, mapping surrogate-predicted physical states to derived and theoretical mass, momentum, and energy quantities, then combining inequality penalties with data loss.}
    \label{fig:surrogate}
\end{figure}

A regression-trained surrogate can fit the data yet still violate conservation laws, especially in OOD regions sought by the optimizer.
We therefore add a differentiable \emph{Physics-Informed Loss} with inequality constraints derived from the steady ablation model.
Concretely, a differentiable mass-from-index operator converts the predicted boundary index into an enclosed-mass estimate.
Steady-ablation budgets then give conservative upper bounds on remaining mass, momentum, and energy.
Squared hinge penalties enforce these bounds.
Valid predictions are left unchanged.
We use upper bounds rather than equality constraints because the budgets are approximate: they depend on simplified models and proxy quantities.
The one-sided hinge preserves inequality semantics.
Full derivations are in Appendix~\ref{sec:appendix:physical_loss}.

\subsection{PPO-based Optimization}
\label{sec:method:ppo}

Our goal in laser pulse optimization is to design an input pulse that maximizes ignition performance while maintaining feasibility. We train the policy using Proximal Policy Optimization (PPO).

\paragraph{Target scope.}
While the VAE decoder and surrogate were pretrained on data with varying target/capsule parameters, the RL optimization uses a single, fixed target configuration.
This matches practical ICF workflows (e.g., Double-Cone Ignition), where target and laser pulses are optimized alternately; fixing the target isolates the pulse shaping problem.

\paragraph{Reward design.}
We define the reward from the Lawson criterion $\rho R T \geq 1.5\;\mathrm{g \cdot keV \cdot cm^{-2}}$ \cite{Atzeni2004PhysicsInertialFusion}. We use the peak $\rho RT$ product over the predicted trajectory as the terminal reward, plus a penalty on laser energy exceeding the experimental budget:
\begin{equation}
r(\mathbf{p}) = \max_{t \in [0, t_{\max}]} \hat \rho \hat R(t) \hat T(t) - \alpha \cdot \max(0, E_{\text{pulse}} - E_{\text{threshold}}).
\end{equation}
The training reward is a $\rho RT$ proxy rather than fusion yield itself, but the two are positively correlated on the optimized pulses: higher $\rho RT$ pushes toward ignition, where self-sustaining alpha-particle heating triggers a nonlinear surge in yield. The ``normalized yield'' reported in Table~\ref{tab:opt_comparison} is MULTI-verified yield $Y$ divided by the experimental baseline $Y_{\mathrm{base}}$, expressed as a percentage.

\paragraph{Bandit formulation and PPO objective.}
Pulse design is cast as a one-step episodic decision process: at each episode the policy $\pi_\theta$ samples a latent action $\mathbf{z}\in\mathbb{R}^{d_z}$, decodes it into a pulse $\mathbf{p}=\mathbf{D}_\psi(\mathbf{z})$ via the fixed VAE decoder, and receives reward $r(\mathbf{p})$. The objective is $\eta(\theta)=\mathbb{E}_{\mathbf{z}\sim \pi_\theta}[r(\mathbf{D}_\psi(\mathbf{z}))]$.
We optimize $\pi_\theta$ with PPO's clipped surrogate objective:
\begin{equation}
\begin{aligned}
\mathcal{L}^{\mathrm{clip}}(\theta)
&=\mathbb{E}_{\mathbf{z}\sim \pi_{\mathrm{old}}}\!\left[
\min\!\Big(
\rho_\theta(\mathbf{z})\,\hat{A}(\mathbf{z}),\right.\\
&\qquad\left.
\mathrm{clip}\!\left(\rho_\theta(\mathbf{z}),1-\epsilon,1+\epsilon\right)
\hat{A}(\mathbf{z})
\Big)\right],
\end{aligned}
\end{equation}
where $\rho_\theta(\mathbf{z})=\pi_\theta(\mathbf{z})/\pi_{\text{old}}(\mathbf{z})$ and the bandit-style advantage is $\hat{A}(\mathbf{z}) = r(\mathbf{D}_\psi(\mathbf{z})) - V_\phi(s)$ with a learned value baseline $V_\phi$.
PPO is preferable to simpler bandit methods because (1) it natively handles high-dimensional continuous action spaces, (2) the clipped objective bounds policy updates and prevents drift into adversarial OOD regions, and (3) the actor-critic architecture reduces variance. Detailed derivations and discussion are provided in Appendix~\ref{sec:appendix:objective}.

\section{Experiments}
\label{sec:experiments}

We evaluate \textbf{Co4ICF} from two perspectives: (i) surrogate accuracy and physical consistency under distribution shift, and (ii) end-to-end laser pulse optimization under a fixed MULTI budget.

\subsection{Experimental Setups}
\label{sec:exp:setup}

\paragraph{Dataset and splits}
We use the MULTI dataset from Section~\ref{sec:method:dataset}. To avoid leakage, we split by shot ID before any augmentation: the VAE and PPCA augmentation models are fit \emph{only} on the \emph{Real} training shots, and test-set shots are never used in any augmentation step. Surrogate results are reported on an ID test split (Real shots held out before augmentation) and an OOD split (Fourier-perturbed shots, generated independently of training data).

\paragraph{Surrogate pretraining and evaluation}
The designer's VAE and the surrogate are trained separately. The VAE learns a latent representation using \emph{PPCA} subset pulses, while the surrogate is pretrained on the union of \emph{Clean} and \emph{PPCA} data. We evaluate on an \textbf{ID} split (\emph{Real}) and an \textbf{OOD} split (\emph{Fourier}), measuring peak MAE and Pearson correlation.

\paragraph{Optimization protocol.}
For pulse optimization, all methods use the same fixed target configuration, 1D surrogate backend, and MULTI labeling budget. Co4ICF performs high-frequency PPO updates on the surrogate and periodically refines the surrogate with 1D-MULTI on-policy labels. Final pulse designs from all optimizers are evaluated by the same direct 2D-MULTI protocol.

\subsection{Surrogate}
\label{sec:exp:surrogate}

\begin{table*}[!t]
\centering
\begin{threeparttable}
\caption{Surrogate architecture comparison under ID and OOD evaluation.}
\label{tab:surrogate_ablation}
\small
\begin{tabular*}{\textwidth}{@{\extracolsep{\fill}}l
                c c  c c  c c
                c c  c c  c c}
\toprule
\textbf{Surrogate} &
\multicolumn{6}{c}{\textbf{ID} (\emph{Real})} &
\multicolumn{6}{c}{\textbf{OOD} (\emph{Fourier})} \\
\cmidrule(lr){2-7}\cmidrule(lr){8-13}
&
\multicolumn{2}{c}{\textbf{Pk $x$}$\downarrow$} &
\multicolumn{2}{c}{\textbf{Pk $y$}$\downarrow$} &
\multicolumn{2}{c}{\textbf{Pearson}$\uparrow$} &

\multicolumn{2}{c}{\textbf{Pk $x$}$\downarrow$} &
\multicolumn{2}{c}{\textbf{Pk $y$}$\downarrow$} &
\multicolumn{2}{c}{\textbf{Pearson}$\uparrow$} \\

\cmidrule(lr){2-3}\cmidrule(lr){4-5}\cmidrule(lr){6-7}
\cmidrule(lr){8-9}\cmidrule(lr){10-11}\cmidrule(lr){12-13}
&
\multicolumn{1}{c}{\textbf{w/}} & \multicolumn{1}{c}{\textbf{w/o}} &
\multicolumn{1}{c}{\textbf{w/}} & \multicolumn{1}{c}{\textbf{w/o}} &
\multicolumn{1}{c}{\textbf{w/}} & \multicolumn{1}{c}{\textbf{w/o}} &
\multicolumn{1}{c}{\textbf{w/}} & \multicolumn{1}{c}{\textbf{w/o}} &
\multicolumn{1}{c}{\textbf{w/}} & \multicolumn{1}{c}{\textbf{w/o}} &
\multicolumn{1}{c}{\textbf{w/}} & \multicolumn{1}{c}{\textbf{w/o}} \\
\midrule

\multicolumn{13}{l}{\textit{Non-autoregressive (NAR)}} \\
\midrule

Vanilla MLP &
\cellcolor{red!15}2.628 & 2.316 &
\cellcolor{green!18}{\underline{0.078}} & {\underline{0.089}} &
\cellcolor{red!15}{\underline{0.971}} & {\underline{0.998}} &
\cellcolor{green!18}9.094 & 10.528 &
\cellcolor{green!18}0.273 & 0.386 &
\cellcolor{green!18}{\underline{0.925}} & {\underline{0.899}} \\
ResMLP &
\cellcolor{green!18}7.431 & 7.922 &
\cellcolor{green!18}0.383 & 0.469 &
\cellcolor{red!15}0.930 & 0.958 &
\cellcolor{green!18}25.798 & 32.128 &
\cellcolor{green!18}1.118 & 1.131 &
\cellcolor{green!18}0.636 & 0.558 \\
1DConv &
\cellcolor{red!15}4.614 & 2.276 &
\cellcolor{red!15}0.181 & 0.095 &
\cellcolor{red!15}0.961 & 0.996 &
\cellcolor{red!15}9.434 & 9.126 &
\cellcolor{green!18}0.431 & 0.482 &
\cellcolor{green!18}0.901 & 0.869 \\
Transformer Enc. &
\cellcolor{red!15}{\underline{2.450}} & {\underline{2.194}} &
\cellcolor{green!18}0.083 & 0.111 &
\cellcolor{red!15}0.969 & 0.997 &
\cellcolor{green!18}{\underline{6.302}} & {\underline{8.020}} &
\cellcolor{green!18}{\underline{0.252}} & {\underline{0.310}} &
\cellcolor{green!18}0.912 & 0.855 \\
\midrule

\multicolumn{13}{l}{\textit{Autoregressive (AR)}} \\
\midrule

Transformer Dec. &
{N/A} & 8.318 &
{N/A} & 0.502 &
{N/A} & 0.762 &
{N/A} & 36.875 &
{N/A} & 1.966 &
{N/A} & 0.451 \\
\bottomrule
\end{tabular*}

\begin{tablenotes}[flushleft]\footnotesize
\item \textbf{w/} and \textbf{w/o}: with / without physics-informed constraints (Section~\ref{sec:method:physics_loss} \& Appendix~\ref{sec:appendix:physical_loss}).
Notably, we do not apply physical constraints to the AR model, since the constraint formulation is not directly comparable between AR and NAR predictors.
\item \textbf{Cell color}: applied to \textbf{w/} only; green = better than w/o, red = worse than w/o (following $\downarrow/\uparrow$).
\item \textbf{Pk $x/y$}: peak $(x,y)$ MAE of the fusion-grain peak location. 
\item \textbf{\underline{Underlined} values}: the best in the column.
\end{tablenotes}
\end{threeparttable}
\end{table*}

Table~\ref{tab:surrogate_ablation} summarizes surrogate performance on ID and OOD splits.

\paragraph{Physics constraints improve OOD}
Adding physics-informed inequality constraints consistently improves generalization on the harder \emph{Fourier} split across NAR backbones.
For the Transformer encoder, OOD peak-$x$ MAE drops from $8.020$ to $6.302$.
Peak-$y$ MAE drops from $0.310$ to $0.252$, and Pearson correlation rises from $0.855$ to $0.912$.
There is a mild trade-off on the ID \emph{Real} split, where conservative bounds slightly reduce in-distribution accuracy; but OOD robustness improves considerably.
The trend holds across architectures.
Physical regularization suppresses non-physical extrapolations and reduces error amplification under optimizer-induced OOD inputs.
Fig.~\ref{fig:surrogate_pretrain_loss_ratio} further tracks the $L_2$ loss ratio during surrogate pretraining.
The curves show that physics regularization stabilizes training and mitigates overfitting, yielding representations that transfer better to unseen pulses.

\begin{figure}[!htbp]
    \centering
    \includegraphics[width=\linewidth]{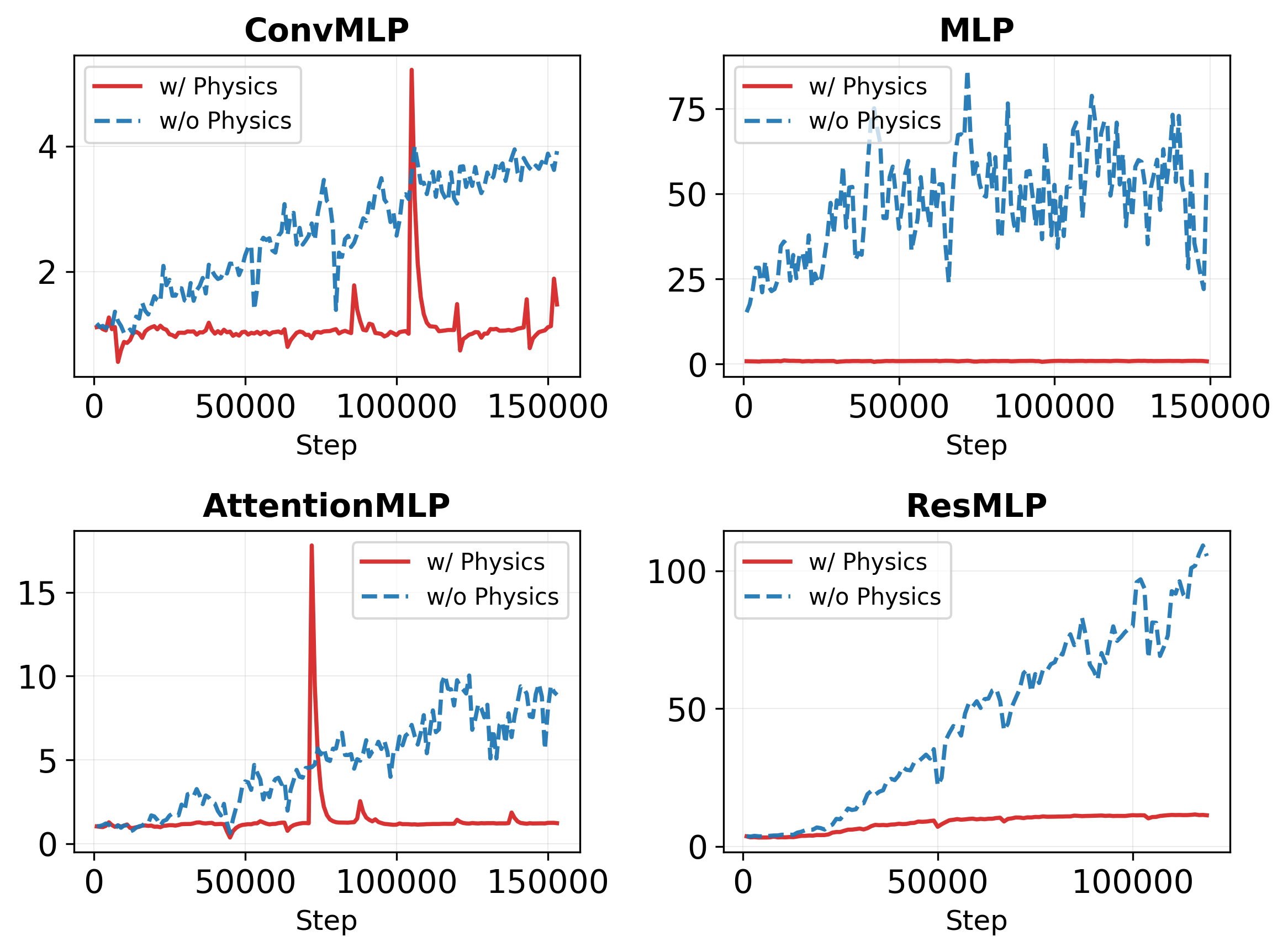}
    \caption{$L_2$ loss ratio during surrogate pretraining.}
    \Description{Four-panel training-curve figure comparing models with and without physics regularization; the physics-regularized curves remain lower and more stable than the non-regularized baselines in most panels.}
    \label{fig:surrogate_pretrain_loss_ratio}
\end{figure}

\paragraph{Transformer encoder is the strongest NAR surrogate.}
It offers the strongest balance between accuracy and robustness among NAR models.
With physics constraints, it has the lowest OOD peak-$x$ MAE ($6.302$), the lowest OOD peak-$y$ MAE ($0.252$), and high OOD Pearson correlation ($0.912$).
Given the short horizon ($T{=}160$) and the need for high throughput, we use it as the default surrogate backend.

\paragraph{NAR vs.\ AR comparison.}
We compare with an AR decoder.
The AR Transformer underperforms on both splits.
Although autoregressive decoding respects temporal causality, it is more sensitive to compounding rollout errors.
Its lower throughput also limits practical training and tuning (Table~\ref{tab:speedup_comparison}).
Our explanation is that the simulator context has only $160$ time steps, which is short relative to typical long-horizon sequence tasks.
With limited context, step-wise temporal modeling helps less, while sequential rollouts still accumulate errors and increase latency.
This agrees with prior forecasting work~\cite{zeng2022transformerseffectivetimeseries}.
In such settings, simple linear or MLP-style baselines can match Transformer variants on many benchmarks.

\paragraph{Surrogate inference speedup.}
Table~\ref{tab:speedup_comparison} reports timing for 1,000 forward evaluations.
MULTI takes 319.83~s on 16 CPU threads.
On CPU, the Transformer encoder takes 14.76~s (21.67$\times$ speedup).
On GPU, it takes 4.58~s (69.83$\times$ speedup).
The Transformer decoder is slower than MULTI on CPU (561.67~s) and only marginally faster on GPU (216.32~s).
The NAR encoder is therefore the most efficient evaluation engine for iterative pulse search.

\begin{table}[!htbp]
\centering
\begin{threeparttable}
\caption{Wall-clock time for 1k forward evaluations.}
\label{tab:speedup_comparison}
\begin{tabular}{l cc}
\toprule
\textbf{Backend} &
\textbf{Time (s)}$\downarrow$ &
\textbf{Speedup} \\
\midrule
Transformer Encoder (CPU) & 14.76  & 21.67 \\
Transformer Encoder (GPU) & {\bfseries 4.58}   & {\bfseries 69.83}\\
Transformer Decoder (CPU) & 561.67 & 0.57  \\
Transformer Decoder (GPU) & 216.32 & 1.48  \\
\midrule
MULTI (CPU)               & 319.83 & 1.00  \\
\bottomrule
\end{tabular}
\begin{tablenotes}[flushleft]\footnotesize
\item Measurements use Ubuntu 24.04 LTS, an AMD Ryzen 9 7940H CPU (16 threads), and an NVIDIA GeForce RTX 4060 Max-Q GPU (8 GB). MULTI uses 16 CPU threads.
\item Speedup is computed as $\text{Time}(\text{MULTI CPU}) / \text{Time}(\text{Backend})$.
\end{tablenotes}
\end{threeparttable}
\end{table}

\subsection{Co-evolving Optimization}

\subsubsection{Main Performance}

We first test whether surrogate refinement improves the 1D pulse-optimization trajectory.
In Fig.~\ref{fig:teaser_ppo}a, periodic refinement reaches 146.1\% after 1D-MULTI re-evaluation.
The frozen-surrogate variant plateaus at 115.4\%.
This gap indicates that updating the surrogate on on-policy samples helps after the policy moves away from the offline training distribution.
We use this 1D trajectory as the main diagnostic in this subsection; final designs are evaluated by direct 2D-MULTI in Section~\ref{sec:exp:coevolving:baselines}.

\begin{figure}[!htbp]
    \centering
    \includegraphics[width=\linewidth]{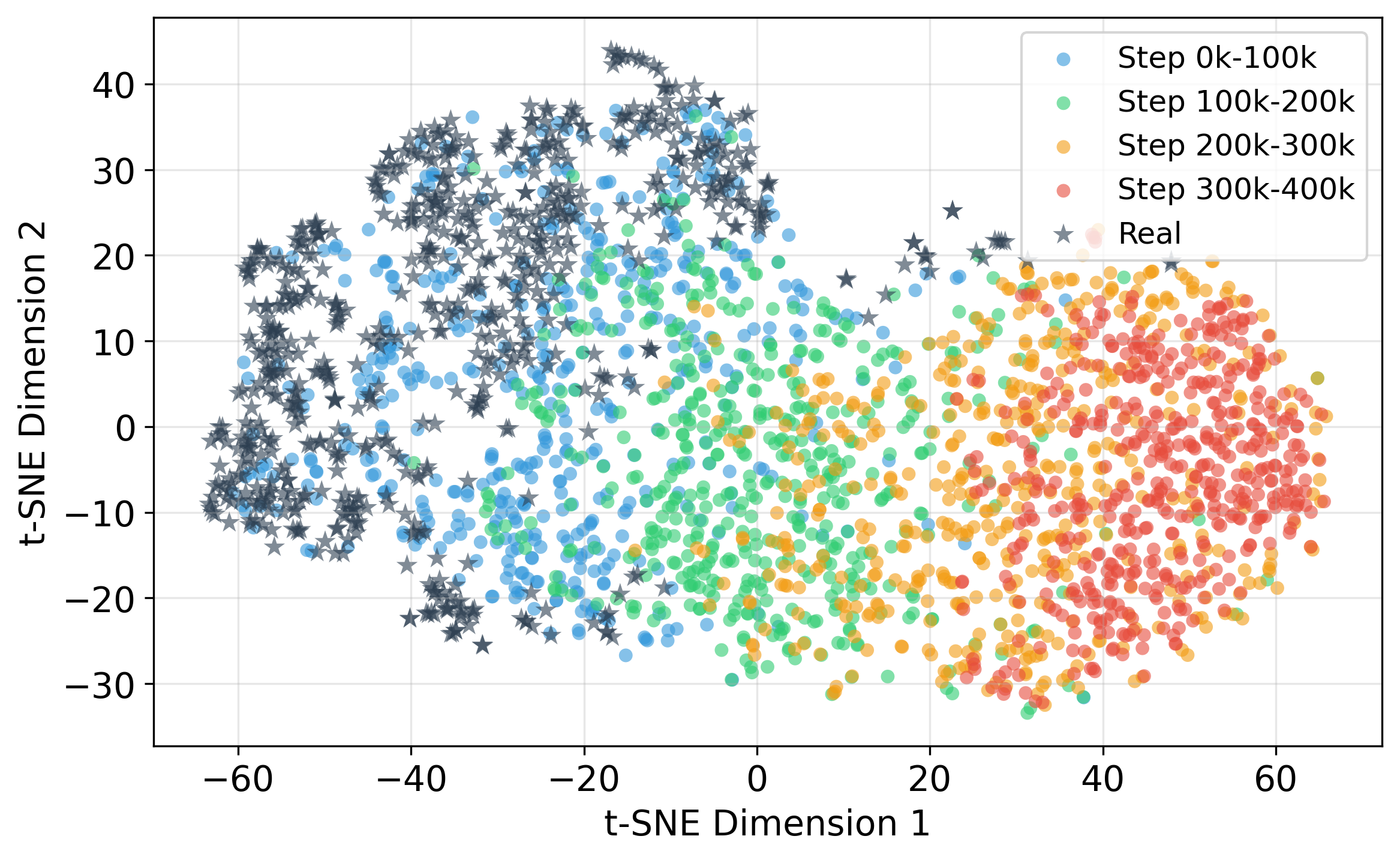}
    \caption{t-SNE visualization of laser pulses during Co4ICF optimization.
    Colored dots: policy-sampled pulses at different stages (0--100k, 100k--200k, 200k--300k, 300k--400k); gray stars: experimental pulses.
    The policy progressively moves beyond the experimental pulse cluster in the embedding space.}
    \Description{t-SNE scatter plot where real pulses appear as gray stars on the left and policy-generated pulses move from blue and green early-stage samples toward orange and red later-stage samples on the right.}
    \label{fig:tsne_coevolve}
\end{figure}

The optimized waveforms provide a qualitative check that the improvement is not driven by isolated reward spikes. Fig.~\ref{fig:teaser_ppo}b compares randomly initialized pulse samples with final policy samples. After RL with surrogate refinement, the designer produces smoother ramp--plateau--shutoff structures among the optimized samples, which is consistent with the higher 1D-MULTI reward curve in Fig.~\ref{fig:teaser_ppo}a.

The policy also moves beyond the measured-pulse cluster within the decoder-defined search space.
Fig.~\ref{fig:tsne_coevolve} embeds policy-sampled 160-step pulses with PCA followed by t-SNE.
Early samples remain close to the Real cluster, whereas later samples expand into broader embedding regions.
This progression suggests that the actor moves beyond imitation while staying constrained by the learned decoder.

\paragraph{Policy training speedup.}
Runtime drops sharply in the PPO loop.
Across 3,967 PPO iterations, MULTI-based training takes 19,800\,s.
The surrogate loop takes 20\,s, giving a \textbf{990$\times$} inner-loop speedup.
This timing covers PPO policy updates only.
It excludes offline data generation, surrogate pretraining, low-frequency relabeling, and 2D-MULTI verification.
Overall speedup depends on cycle count and labeling budget $M$.
The main saving is replacing per-iteration MULTI calls with surrogate queries.

\subsubsection{Direct 2D Evaluation and Ablation}
\label{sec:exp:coevolving:baselines}

We next test whether 1D-optimized pulses remain strong under direct 2D-MULTI evaluation.
All optimizers use the same 1D surrogate and MULTI labeling budget during search.
2D-MULTI is used only for post-hoc evaluation of final designs.
Thus, no 2D samples are used to train, fine-tune, or select the PPO policy.

\begin{table}[!htbp]
\centering
\caption{Direct 2D-MULTI evaluation of final pulse designs under the same 1D surrogate and MULTI labeling budget. All reported physical quantities and normalized yields are computed by 2D-MULTI, which is used only for post-hoc evaluation. Normalized Yield is the 2D-MULTI fusion yield divided by the baseline pulse's 2D-MULTI yield.}
\label{tab:opt_comparison}
\small
\begin{tabular}{l cccc}
\toprule
\textbf{Case} & $\bar{\rho}$ & $\rho R$ & $T$ & \textbf{Norm. Yield}$\uparrow$ \\
\midrule
Baseline & 7.60 & 0.907 & 163.63 & 100.0\% \\
BO & 7.20 & 1.100 & 233.98 & 173.5\% \\
GA & 7.30 & 0.829 & 218.10 & 121.9\% \\
W/O Update (Static) & 6.52 & 0.482 & 388.72 & 126.3\% \\
W/O Update (Enlarged) & 6.49 & 0.512 & 422.09 & 174.8\% \\
W/ Update (Co4ICF) & \textbf{7.10} & \textbf{1.063} & 344.60 & \textbf{246.9\%} \\
\bottomrule
\end{tabular}
\end{table}

Table~\ref{tab:opt_comparison} shows that Co4ICF achieves the highest normalized 2D yield: 246.9\%, compared with 173.5\% for BO, 121.9\% for GA, and 126.3\% for frozen-surrogate PPO.
The table also gives a budget-matched ablation for the co-evolving surrogate.
W/O Update (Enlarged) retrains the frozen surrogate with the same number of extra samples as Co4ICF, but draws them from the original data distribution instead of the current policy.
This enlarged frozen surrogate improves over W/O Update (Static), from 126.3\% to 174.8\%, but remains well below W/ Update (Co4ICF).
Together, these comparisons suggest that the gain comes from coupling optimization with on-policy surrogate refinement, rather than from optimizer choice or extra data alone.

\subsubsection{Distribution-Shift Diagnostics}
\label{sec:dist_shift}

The budget-matched ablation shows the final effect of on-policy refinement. To diagnose this mechanism, we track the peak error between surrogate predictions and MULTI-evaluated trajectories during PPO updates. This diagnostic measures whether the surrogate remains reliable as the policy shifts the input distribution away from the offline training set.
\begin{figure}[!htbp]
    \centering
    \includegraphics[width=\linewidth]{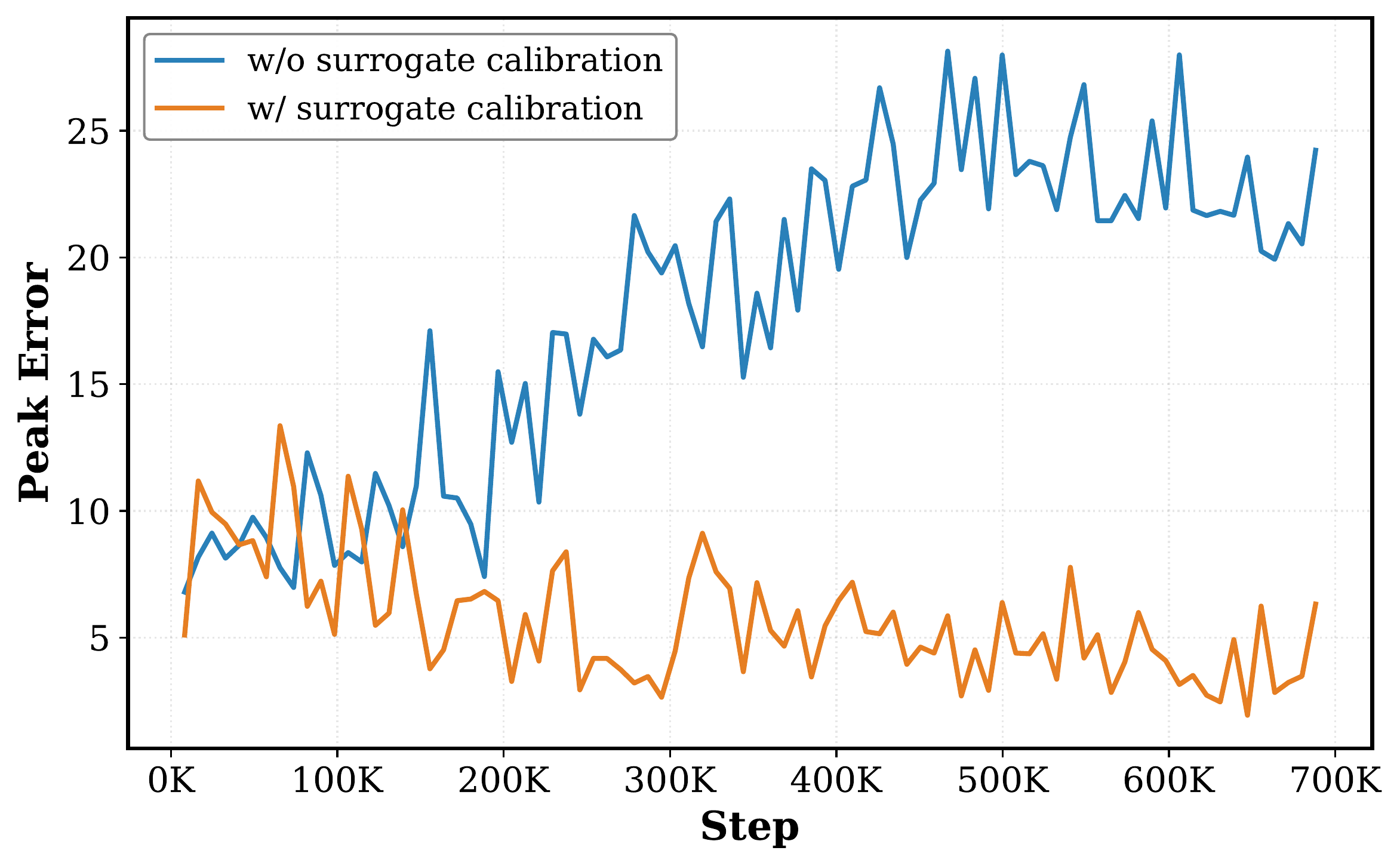}
    \caption{Surrogate--MULTI peak error over PPO training steps. Without calibration, the error grows unbounded as the policy drives inputs OOD; with periodic calibration, the error stays bounded.}
    \Description{Line plot comparing surrogate-MULTI peak error over PPO training steps: the uncalibrated curve climbs from the single digits into the twenties, while the calibrated curve stays mostly below ten.}
    \label{fig:peak_error}
\end{figure}

Fig.~\ref{fig:peak_error} shows that the frozen surrogate becomes increasingly unreliable during optimization.
Its peak error grows and eventually spikes, indicating optimizer-induced distribution shift.
Periodic calibration keeps the surrogate--MULTI gap bounded, explaining why W/ Update outperforms both W/O Update variants in Table~\ref{tab:opt_comparison}.

\section{Conclusions}
\label{sec:conclusion}

We presented \textbf{Co4ICF}, a co-evolving framework that couples a physics-informed surrogate with an RL-based pulse designer for ICF design under limited high-fidelity budgets.
The dual-frequency loop combines high-frequency PPO updates with low-frequency surrogate fine-tuning to correct local extrapolation errors and reduce reward hacking.
Co4ICF achieves 246.9\% normalized yield in direct 2D-MULTI evaluation relative to the experimental baseline.
It does so without 2D training or fine-tuning, while retaining a 990$\times$ inner-loop speedup.
Budget-matched ablations indicate that these gains are not explained solely by additional simulation data.

More broadly, Co4ICF suggests a recipe for simulation-informed design.
Treat optimizer-induced shift as a training signal, and keep the learned simulator aligned with decision-relevant regions. 

\section{Limitations and Broader Impact}

\paragraph{Limitations}
The framework relies on MULTI-IFE (1D \& 2D) for high-fidelity feedback \cite{ramis2016multi}, so it cannot capture asymmetric instabilities (e.g., Rayleigh--Taylor) that matter in real 3D implosions. Search is constrained to the VAE manifold, which keeps pulses feasible but risks missing physically valid shapes outside it.

\paragraph{Broader impact.}
The co-evolving paradigm is not specific to ICF and may transfer to other domains where iterative optimization interacts with a learned model (e.g., molecular design, materials discovery). 
However, ICF research inherently bridges civilian fusion energy and weapons physics; methodologies for optimizing implosion performance could, in principle, be applied to weapons-relevant contexts. 
Researchers adapting this paradigm should adhere to relevant export control and non-proliferation frameworks.


\begin{acks}
This work was supported by the AI for Science Program, Shanghai Municipal Commission of Economy and Informatization (2025-GZL-RGZN-BTBX-02024), and National Natural Science Foundation of China Program (125B100032, 92370201).
\end{acks}

\balance
\bibliographystyle{ACM-Reference-Format}
\bibliography{sample-base}


\begin{thebibliography}{25}


\ifx \showCODEN    \undefined \def \showCODEN     #1{\unskip}     \fi
\ifx \showISBNx    \undefined \def \showISBNx     #1{\unskip}     \fi
\ifx \showISBNxiii \undefined \def \showISBNxiii  #1{\unskip}     \fi
\ifx \showISSN     \undefined \def \showISSN      #1{\unskip}     \fi
\ifx \showLCCN     \undefined \def \showLCCN      #1{\unskip}     \fi
\ifx \shownote     \undefined \def \shownote      #1{#1}          \fi
\ifx \showarticletitle \undefined \def \showarticletitle #1{#1}   \fi
\ifx \showURL      \undefined \def \showURL       {\relax}        \fi
\providecommand\bibfield[2]{#2}
\providecommand\bibinfo[2]{#2}
\providecommand\natexlab[1]{#1}
\providecommand\showeprint[2][]{arXiv:#2}

\bibitem[Anirudh et~al\mbox{.}(2020)]%
        {anirudh2020improved}
\bibfield{author}{\bibinfo{person}{Rushil Anirudh},
  \bibinfo{person}{Jayaraman~J Thiagarajan}, \bibinfo{person}{Peer-Timo
  Bremer}, {and} \bibinfo{person}{Brian~K Spears}.}
  \bibinfo{year}{2020}\natexlab{}.
\newblock \showarticletitle{Improved surrogates in inertial confinement fusion
  with manifold and cycle consistencies}.
\newblock \bibinfo{journal}{\emph{Proceedings of the National Academy of
  Sciences}} \bibinfo{volume}{117}, \bibinfo{number}{18}
  (\bibinfo{year}{2020}), \bibinfo{pages}{9741--9746}.
\newblock


\bibitem[Atzeni and ter Vehn(2004)]%
        {Atzeni2004PhysicsInertialFusion}
\bibfield{author}{\bibinfo{person}{Stefano Atzeni} {and}
  \bibinfo{person}{J{\"u}rgen~Meyer ter Vehn}.}
  \bibinfo{year}{2004}\natexlab{}.
\newblock \bibinfo{booktitle}{\emph{The Physics of Inertial Fusion:
  Beam--Plasma Interaction, Hydrodynamics, Hot Dense Matter}}.
\newblock \bibinfo{publisher}{Oxford University Press},
  \bibinfo{address}{Oxford, UK}.
\newblock
\showISBNx{9780198562641}
\href{https://doi.org/10.1093/acprof:oso/9780198562641.001.0001}{doi:\nolinkurl{10.1093/acprof:oso/9780198562641.001.0001}}


\bibitem[Capuano et~al\mbox{.}(2025)]%
        {capuano2025shaping}
\bibfield{author}{\bibinfo{person}{Francesco Capuano}, \bibinfo{person}{Davorin
  Peceli}, {and} \bibinfo{person}{Gabriele Tiboni}.}
  \bibinfo{year}{2025}\natexlab{}.
\newblock \showarticletitle{Shaping Laser Pulses with Reinforcement Learning}.
\newblock \bibinfo{journal}{\emph{arXiv preprint arXiv:2503.00499}}
  (\bibinfo{year}{2025}).
\newblock


\bibitem[Degrave et~al\mbox{.}(2022)]%
        {degrave2022magnetic}
\bibfield{author}{\bibinfo{person}{Jonas Degrave}, \bibinfo{person}{Federico
  Felici}, \bibinfo{person}{Jonas Buchli}, \bibinfo{person}{Michael Neunert},
  \bibinfo{person}{Brendan Tracey}, \bibinfo{person}{Francesco Carpanese},
  \bibinfo{person}{Timo Ewalds}, \bibinfo{person}{Roland Hafner},
  \bibinfo{person}{Abbas Abdolmaleki}, \bibinfo{person}{Diego de Las~Casas},
  {et~al\mbox{.}}} \bibinfo{year}{2022}\natexlab{}.
\newblock \showarticletitle{Magnetic control of tokamak plasmas through deep
  reinforcement learning}.
\newblock \bibinfo{journal}{\emph{Nature}} \bibinfo{volume}{602},
  \bibinfo{number}{7897} (\bibinfo{year}{2022}), \bibinfo{pages}{414--419}.
\newblock


\bibitem[Ejaz et~al\mbox{.}(2024)]%
        {ejaz2024deep}
\bibfield{author}{\bibinfo{person}{Rahman Ejaz}, \bibinfo{person}{Varchas
  Gopalaswamy}, \bibinfo{person}{A Lees}, \bibinfo{person}{C Kanan},
  \bibinfo{person}{D Cao}, {and} \bibinfo{person}{R Betti}.}
  \bibinfo{year}{2024}\natexlab{}.
\newblock \showarticletitle{Deep learning-based predictive models for laser
  direct drive at the Omega Laser Facility}.
\newblock \bibinfo{journal}{\emph{Physics of Plasmas}} \bibinfo{volume}{31},
  \bibinfo{number}{5} (\bibinfo{year}{2024}).
\newblock


\bibitem[Gopalaswamy et~al\mbox{.}(2025)]%
        {gopalaswamy2025automated}
\bibfield{author}{\bibinfo{person}{V Gopalaswamy}, \bibinfo{person}{A Lees},
  \bibinfo{person}{R Ejaz}, \bibinfo{person}{CA Thomas}, \bibinfo{person}{TJB
  Collins}, \bibinfo{person}{KS Anderson}, \bibinfo{person}{W Ebmeyer}, {and}
  \bibinfo{person}{R Betti}.} \bibinfo{year}{2025}\natexlab{}.
\newblock \showarticletitle{Automated and highly parallelized Bayesian
  optimization scheme for direct drive fusion experiments on OMEGA}.
\newblock \bibinfo{journal}{\emph{Physical Review Research}}
  \bibinfo{volume}{7}, \bibinfo{number}{1} (\bibinfo{year}{2025}),
  \bibinfo{pages}{013009}.
\newblock


\bibitem[Humbird(2019)]%
        {humbird2019machine}
\bibfield{author}{\bibinfo{person}{Kelli~Denise Humbird}.}
  \bibinfo{year}{2019}\natexlab{}.
\newblock \emph{\bibinfo{title}{Machine learning guided discovery and design
  for inertial confinement fusion}}.
\newblock \bibinfo{thesistype}{Ph.\,D. Dissertation}. \bibinfo{school}{Texas
  A\&M University}.
\newblock


\bibitem[Humbird et~al\mbox{.}(2021)]%
        {humbird2021cognitive}
\bibfield{author}{\bibinfo{person}{Kelli~D Humbird}, \bibinfo{person}{J~Luc
  Peterson}, \bibinfo{person}{J Salmonson}, {and} \bibinfo{person}{Brian~K
  Spears}.} \bibinfo{year}{2021}\natexlab{}.
\newblock \showarticletitle{Cognitive simulation models for inertial
  confinement fusion: Combining simulation and experimental data}.
\newblock \bibinfo{journal}{\emph{Physics of Plasmas}} \bibinfo{volume}{28},
  \bibinfo{number}{4} (\bibinfo{year}{2021}).
\newblock


\bibitem[Hurricane(2024)]%
        {hurricane2024ignition}
\bibfield{author}{\bibinfo{person}{Omar~A Hurricane}.}
  \bibinfo{year}{2024}\natexlab{}.
\newblock \showarticletitle{How ignition and target gain> 1 were achieved in
  inertial fusion}.
\newblock \bibinfo{journal}{\emph{High Energy Density Physics}}
  \bibinfo{volume}{53} (\bibinfo{year}{2024}), \bibinfo{pages}{101157}.
\newblock


\bibitem[Karniadakis et~al\mbox{.}(2021)]%
        {karniadakis2021physics}
\bibfield{author}{\bibinfo{person}{George~Em Karniadakis},
  \bibinfo{person}{Ioannis~G Kevrekidis}, \bibinfo{person}{Lu Lu},
  \bibinfo{person}{Paris Perdikaris}, \bibinfo{person}{Sifan Wang}, {and}
  \bibinfo{person}{Liu Yang}.} \bibinfo{year}{2021}\natexlab{}.
\newblock \showarticletitle{Physics-informed machine learning}.
\newblock \bibinfo{journal}{\emph{Nature Reviews Physics}} \bibinfo{volume}{3},
  \bibinfo{number}{6} (\bibinfo{year}{2021}), \bibinfo{pages}{422--440}.
\newblock


\bibitem[Kim et~al\mbox{.}(2025)]%
        {kim2024offline}
\bibfield{author}{\bibinfo{person}{Minsu Kim}, \bibinfo{person}{Jiayao Gu},
  \bibinfo{person}{Ye Yuan}, \bibinfo{person}{Taeyoung Yun},
  \bibinfo{person}{Zixuan Liu}, \bibinfo{person}{Yoshua Bengio}, {and}
  \bibinfo{person}{Can Chen}.} \bibinfo{year}{2025}\natexlab{}.
\newblock \bibinfo{title}{Offline Model-Based Optimization: Comprehensive
  Review}.
\newblock
\showeprint[arxiv]{2503.17286}~[cs.LG]
\urldef\tempurl%
\url{https://arxiv.org/abs/2503.17286}
\showURL{%
\tempurl}


\bibitem[Li et~al\mbox{.}(2023)]%
        {li2023hybrid}
\bibfield{author}{\bibinfo{person}{Z Li}, \bibinfo{person}{ZQ Zhao},
  \bibinfo{person}{XH Yang}, \bibinfo{person}{GB Zhang}, \bibinfo{person}{YY
  Ma}, \bibinfo{person}{H Xu}, \bibinfo{person}{FY Wu}, \bibinfo{person}{FQ
  Shao}, {and} \bibinfo{person}{J Zhang}.} \bibinfo{year}{2023}\natexlab{}.
\newblock \showarticletitle{Hybrid optimization of laser-driven fusion targets
  and laser profiles}.
\newblock \bibinfo{journal}{\emph{Plasma Physics and Controlled Fusion}}
  \bibinfo{volume}{66}, \bibinfo{number}{1} (\bibinfo{year}{2023}),
  \bibinfo{pages}{015010}.
\newblock


\bibitem[Liu et~al\mbox{.}(2024)]%
        {liu2024saea}
\bibfield{author}{\bibinfo{person}{Shulei Liu}, \bibinfo{person}{Handing Wang},
  \bibinfo{person}{Wei Peng}, {and} \bibinfo{person}{Wen Yao}.}
  \bibinfo{year}{2024}\natexlab{}.
\newblock \showarticletitle{Surrogate-assisted evolutionary algorithms for
  expensive combinatorial optimization: a survey}.
\newblock \bibinfo{journal}{\emph{Complex \& Intelligent Systems}}
  \bibinfo{volume}{10}, \bibinfo{number}{4} (\bibinfo{date}{Aug.}
  \bibinfo{year}{2024}), \bibinfo{pages}{5933--5949}.
\newblock
\showISSN{2198-6053}
\href{https://doi.org/10.1007/s40747-024-01465-5}{doi:\nolinkurl{10.1007/s40747-024-01465-5}}


\bibitem[Olson et~al\mbox{.}(2024)]%
        {olson2024transformer}
\bibfield{author}{\bibinfo{person}{Matthew~L Olson}, \bibinfo{person}{Shusen
  Liu}, \bibinfo{person}{Jayaraman~J Thiagarajan}, \bibinfo{person}{Bogdan
  Kustowski}, \bibinfo{person}{Weng-Keen Wong}, {and} \bibinfo{person}{Rushil
  Anirudh}.} \bibinfo{year}{2024}\natexlab{}.
\newblock \showarticletitle{Transformer-powered surrogates close the ICF
  simulation-experiment gap with extremely limited data}.
\newblock \bibinfo{journal}{\emph{Machine Learning: Science and Technology}}
  \bibinfo{volume}{5}, \bibinfo{number}{2} (\bibinfo{year}{2024}),
  \bibinfo{pages}{025054}.
\newblock


\bibitem[Raissi et~al\mbox{.}(2019)]%
        {raissi2019physics}
\bibfield{author}{\bibinfo{person}{Maziar Raissi}, \bibinfo{person}{Paris
  Perdikaris}, {and} \bibinfo{person}{George~E Karniadakis}.}
  \bibinfo{year}{2019}\natexlab{}.
\newblock \showarticletitle{Physics-informed neural networks: A deep learning
  framework for solving forward and inverse problems involving nonlinear
  partial differential equations}.
\newblock \bibinfo{journal}{\emph{Journal of Computational physics}}
  \bibinfo{volume}{378} (\bibinfo{year}{2019}), \bibinfo{pages}{686--707}.
\newblock


\bibitem[Ramis and Meyer-ter Vehn(2016)]%
        {ramis2016multi}
\bibfield{author}{\bibinfo{person}{Rafael Ramis} {and}
  \bibinfo{person}{J{\"u}rgen Meyer-ter Vehn}.}
  \bibinfo{year}{2016}\natexlab{}.
\newblock \showarticletitle{MULTI-IFE—A one-dimensional computer code for
  Inertial Fusion Energy (IFE) target simulations}.
\newblock \bibinfo{journal}{\emph{Computer Physics Communications}}
  \bibinfo{volume}{203} (\bibinfo{year}{2016}), \bibinfo{pages}{226--237}.
\newblock


\bibitem[Spears et~al\mbox{.}(2025)]%
        {spears2025predicting}
\bibfield{author}{\bibinfo{person}{Brian~K Spears}, \bibinfo{person}{Scott
  Brandon}, \bibinfo{person}{Dan~T Casey}, \bibinfo{person}{John~E Field},
  \bibinfo{person}{Jim~A Gaffney}, \bibinfo{person}{Kelli~D Humbird},
  \bibinfo{person}{Andrea~L Kritcher}, \bibinfo{person}{Michael~KG Kruse},
  \bibinfo{person}{Eugene Kur}, \bibinfo{person}{Bogdan Kustowski},
  {et~al\mbox{.}}} \bibinfo{year}{2025}\natexlab{}.
\newblock \showarticletitle{Predicting fusion ignition at the National Ignition
  Facility with physics-informed deep learning}.
\newblock \bibinfo{journal}{\emph{Science}} \bibinfo{volume}{389},
  \bibinfo{number}{6761} (\bibinfo{year}{2025}), \bibinfo{pages}{727--731}.
\newblock


\bibitem[Sutton(1991)]%
        {sutton1991dyna}
\bibfield{author}{\bibinfo{person}{Richard~S Sutton}.}
  \bibinfo{year}{1991}\natexlab{}.
\newblock \showarticletitle{Dyna, an integrated architecture for learning,
  planning, and reacting}.
\newblock \bibinfo{journal}{\emph{ACM SIGART Bulletin}} \bibinfo{volume}{2},
  \bibinfo{number}{4} (\bibinfo{year}{1991}), \bibinfo{pages}{160--163}.
\newblock


\bibitem[Tao et~al\mbox{.}(2023)]%
        {tao2023laser}
\bibfield{author}{\bibinfo{person}{Tao Tao}, \bibinfo{person}{Guannan Zheng},
  \bibinfo{person}{Qing Jia}, \bibinfo{person}{Rui Yan}, {and}
  \bibinfo{person}{Jian Zheng}.} \bibinfo{year}{2023}\natexlab{}.
\newblock \showarticletitle{Laser pulse shape designer for direct-drive
  inertial confinement fusion implosions}.
\newblock \bibinfo{journal}{\emph{High Power Laser Science and Engineering}}
  \bibinfo{volume}{11} (\bibinfo{year}{2023}), \bibinfo{pages}{e41}.
\newblock


\bibitem[Trabucco et~al\mbox{.}(2022)]%
        {trabucco2022design}
\bibfield{author}{\bibinfo{person}{Brandon Trabucco}, \bibinfo{person}{Xinyang
  Geng}, \bibinfo{person}{Aviral Kumar}, {and} \bibinfo{person}{Sergey
  Levine}.} \bibinfo{year}{2022}\natexlab{}.
\newblock \bibinfo{title}{Design-Bench: Benchmarks for Data-Driven Offline
  Model-Based Optimization}.
\newblock
\showeprint[arxiv]{2202.08450}~[cs.LG]
\urldef\tempurl%
\url{https://arxiv.org/abs/2202.08450}
\showURL{%
\tempurl}


\bibitem[Wang et~al\mbox{.}(2024)]%
        {wang2024multifidelity}
\bibfield{author}{\bibinfo{person}{Jingyi Wang}, \bibinfo{person}{N Chiang},
  \bibinfo{person}{Andrew Gillette}, {and} \bibinfo{person}{J~Luc Peterson}.}
  \bibinfo{year}{2024}\natexlab{}.
\newblock \showarticletitle{A multifidelity Bayesian optimization method for
  inertial confinement fusion design}.
\newblock \bibinfo{journal}{\emph{Physics of Plasmas}} \bibinfo{volume}{31},
  \bibinfo{number}{3} (\bibinfo{year}{2024}).
\newblock


\bibitem[Wang et~al\mbox{.}(2025)]%
        {wang2025predictive}
\bibfield{author}{\bibinfo{person}{Zixu Wang}, \bibinfo{person}{Yuhan Wang},
  \bibinfo{person}{Junfei Ma}, \bibinfo{person}{Fuyuan Wu},
  \bibinfo{person}{Junchi Yan}, \bibinfo{person}{Xiaohui Yuan},
  \bibinfo{person}{Zhe Zhang}, {and} \bibinfo{person}{Jie Zhang}.}
  \bibinfo{year}{2025}\natexlab{}.
\newblock \showarticletitle{Predictive Hydrodynamic Simulations for Laser
  Direct-drive Implosion Experiments via Artificial Intelligence}.
\newblock \bibinfo{journal}{\emph{arXiv preprint arXiv:2507.16227}}
  (\bibinfo{year}{2025}).
\newblock


\bibitem[Wei et~al\mbox{.}(2024)]%
        {wei2024machine}
\bibfield{author}{\bibinfo{person}{S Wei}, \bibinfo{person}{F Wu},
  \bibinfo{person}{Y Zhu}, \bibinfo{person}{J Yang}, \bibinfo{person}{L Zeng},
  \bibinfo{person}{X Li}, {and} \bibinfo{person}{J Zhang}.}
  \bibinfo{year}{2024}\natexlab{}.
\newblock \showarticletitle{A Machine Learning Method for the Optimization
  Design of Laser Pulse in Fast Ignition Simulations}.
\newblock \bibinfo{journal}{\emph{Journal of Fusion Energy}}
  \bibinfo{volume}{43}, \bibinfo{number}{1} (\bibinfo{year}{2024}),
  \bibinfo{pages}{6}.
\newblock


\bibitem[Wu et~al\mbox{.}(2022)]%
        {wu2022machine}
\bibfield{author}{\bibinfo{person}{Fuyuan Wu}, \bibinfo{person}{Xiaohu Yang},
  \bibinfo{person}{Yanyun Ma}, \bibinfo{person}{Qi Zhang}, \bibinfo{person}{Zhe
  Zhang}, \bibinfo{person}{Xiaohui Yuan}, \bibinfo{person}{Hao Liu},
  \bibinfo{person}{Zhengdong Liu}, \bibinfo{person}{Jiayong Zhong},
  \bibinfo{person}{Jian Zheng}, {et~al\mbox{.}}}
  \bibinfo{year}{2022}\natexlab{}.
\newblock \showarticletitle{Machine-learning guided optimization of laser
  pulses for direct-drive implosions}.
\newblock \bibinfo{journal}{\emph{High Power Laser Science and Engineering}}
  \bibinfo{volume}{10} (\bibinfo{year}{2022}), \bibinfo{pages}{e12}.
\newblock


\bibitem[Zeng et~al\mbox{.}(2022)]%
        {zeng2022transformerseffectivetimeseries}
\bibfield{author}{\bibinfo{person}{Ailing Zeng}, \bibinfo{person}{Muxi Chen},
  \bibinfo{person}{Lei Zhang}, {and} \bibinfo{person}{Qiang Xu}.}
  \bibinfo{year}{2022}\natexlab{}.
\newblock \bibinfo{title}{Are Transformers Effective for Time Series
  Forecasting?}
\newblock
\showeprint[arxiv]{2205.13504}~[cs.AI]
\urldef\tempurl%
\url{https://arxiv.org/abs/2205.13504}
\showURL{%
\tempurl}


\end{thebibliography}


\appendix

\section{Dataset and Data Augmentation}
\label{sec:appendix:data}
This section describes the dataset composition and the three data augmentation strategies (VAE, PPCA, Fourier-domain noise) used to expand training coverage beyond the limited set of experimentally measured pulses.

\subsection{Dataset Overview}
The dataset contains about $90$k normalized pairs, stored in HDF5 format by shot ID.
Each quantity is interpolated onto a uniform grid of 160 time steps (0.05~ns resolution) and stored as a fixed-length vector.
Each input consists of a laser power pulse and three discrete target parameters.

\begin{figure}[htbp]
    \centering
    \includegraphics[width=\linewidth]{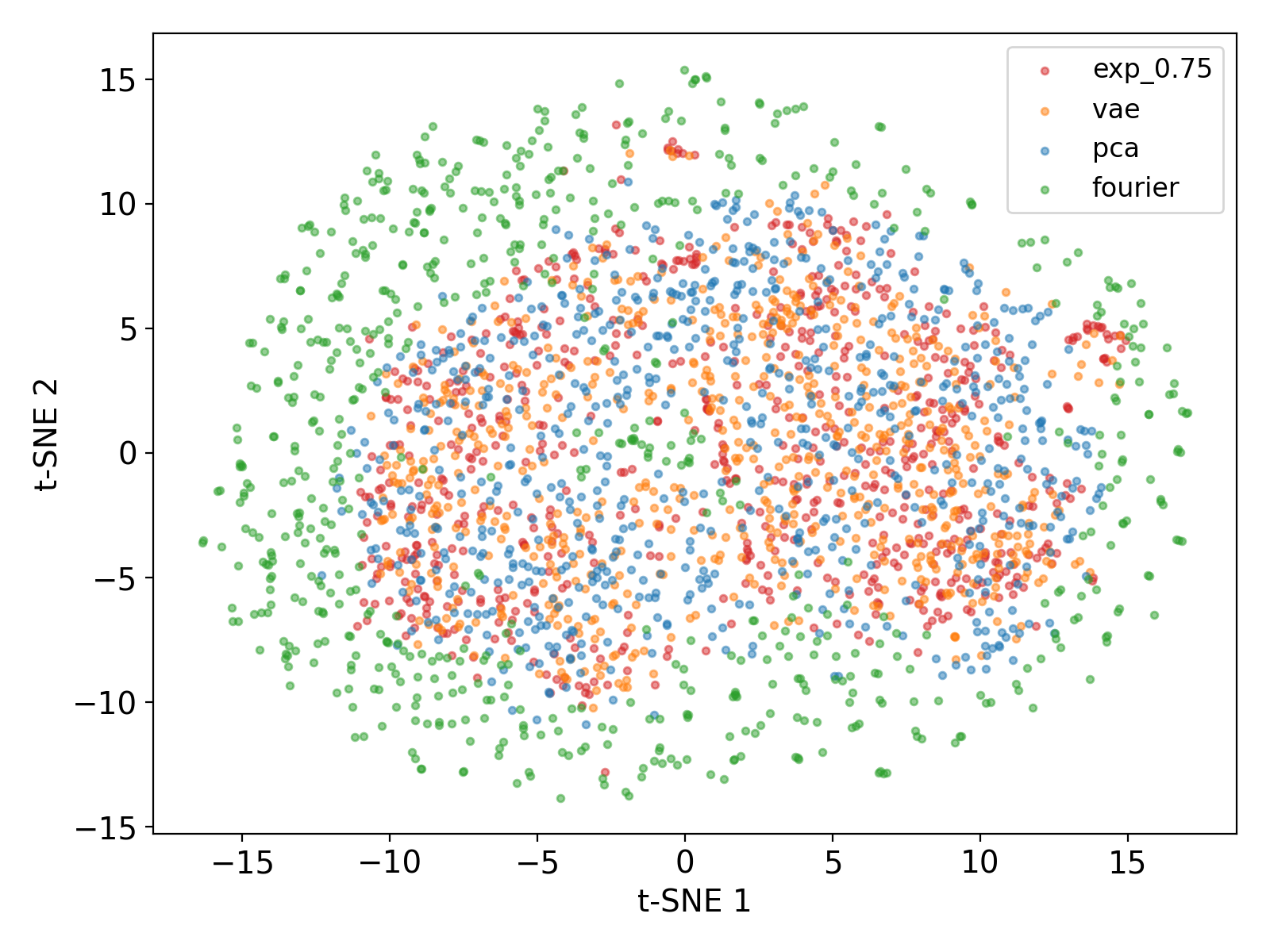}
    \caption{t-SNE visualization of input laser pulses across dataset subsets. Coverage increases from Real (701 pulses) to VAE-augmented, PPCA, and Fourier-noise samples.}
    \Description{t-SNE scatter plot of dataset subsets, with experimental, VAE, and PPCA samples concentrated in the center and Fourier-augmented samples spread over a wider surrounding region.}
    \label{fig:tsne_pulses}
\end{figure}

\subsection{VAE Sampling}
\label{sec:appendix:data:vae_sampling}
We employ a variational autoencoder (VAE) to learn a compact generative representation of measured laser pulses for offline data augmentation (\texttt{VAE\_aug}).
\texttt{VAE\_aug} is trained \emph{only} on the \emph{Real} subset, producing samples that follow the same empirical distribution as measured pulses.
This dataset VAE is \emph{not} the VAE used in the downstream optimization stage.

\subsubsection{Model and Training Objective}
Let $p \in \mathbb{R}^{T}$ denote a normalized laser power pulse interpolated onto a uniform grid with $T=160$ time steps.
The VAE consists of an encoder $q_\phi(z|p)$ and a decoder $p_\theta(p|z)$, where $z\in\mathbb{R}^{d}$ is a latent code.
We assume a standard Gaussian prior $p(z)=\mathcal{N}(0,I)$ and train the VAE by maximizing the evidence lower bound (ELBO):
\begin{equation}
\mathcal{L}_{\mathrm{VAE}}(p;\phi,\theta)
=
\mathbb{E}_{z\sim q_\phi(z\mid p)}\big[ \|p-\hat{p}\|_2^2 \big]
+
\beta\,\mathrm{KL}\left(q_\phi(z\mid p)\|,\mathcal{N}(0,I)\right),
\label{eq:vae_elbo}
\end{equation}
\begin{equation}
    \hat{p}=D_\theta(z),
\end{equation}
where $D_\theta(\cdot)$ denotes the decoder, and $\beta$ is a weighting factor controlling the strength of the KL regularization.

\subsubsection{Sampling Procedure}
After pretraining, we draw latent vectors from the prior and decode them:
\begin{equation}
z \sim \mathcal{N}(0,I),\qquad
\tilde{p} = \Pi_{\ge 0}\!\left(D_\theta(z)\right),
\label{eq:vae_sampling}
\end{equation}
where $\Pi_{\ge 0}(\cdot)$ denotes a non-negativity projection (implemented as $\max(\cdot,0)$) to ensure physically valid laser power.

\subsection{PPCA Sampling}
\label{sec:appendix:data:ppca_sampling}
To provide a linear generative baseline complementary to the non-linear VAE, we utilize Probabilistic Principal Component Analysis (PPCA).
Similar to the VAE strategy, the PPCA model is fitted exclusively on the \emph{Real} subset.

\subsubsection{Model and Training Objective}
We model the observed pulse $p \in \mathbb{R}^{T}$ as a linear transformation of a lower-dimensional latent variable $z \in \mathbb{R}^{d_{pca}}$ with additive Gaussian noise:
\begin{equation}
p = W z + \mu + \epsilon,
\label{eq:ppca_model}
\end{equation}
where $\mu \in \mathbb{R}^{T}$ is the data mean, $W \in \mathbb{R}^{T \times d_{pca}}$ is the factor loading matrix, and $\epsilon \sim \mathcal{N}(0, \sigma^2 I)$ represents isotropic noise.

\subsubsection{Sampling Procedure}
\begin{equation}
z \sim \mathcal{N}(0,I), \qquad
\tilde{p} = \Pi_{\ge 0}\!\left( W z + \mu + \epsilon_{\mathrm{sample}} \right),
\label{eq:ppca_sampling}
\end{equation}
where $\epsilon_{\mathrm{sample}}$ is drawn from $\mathcal{N}(0, \sigma^2 I)$, and $\Pi_{\ge 0}(\cdot)$ enforces the physical non-negativity constraint.

\subsection{Fourier-Domain Exponential-Decay Noise Augmentation}
\label{sec:appendix:data:fourier_noise}

We perturb measured pulses in the Fourier domain to generate OOD yet physically plausible laser pulses.
For a real-valued input pulse $p(t)\in\mathbb{R}^{T}$, we compute its one-sided real FFT
\begin{equation}
P_k=\mathrm{RFFT}(p)_k\in\mathbb{C}
\end{equation}
then inject additive complex Gaussian noise whose magnitude decays exponentially with frequency index:
\begin{equation}
\tilde{P}_k = P_k + A\exp\left(-\frac{k}{\tau}\right)\epsilon_k,\quad
\epsilon_k \sim \mathcal{N}(\mu, 1)+ i\mathcal{N}(\mu,1)
\end{equation}
where $k$ is the discrete frequency index, $\tau$ is the decay factor, and $A$ sets the overall perturbation strength.
The exponential envelope biases perturbations toward low-frequency components.
This largely preserves coarse pulse structure while producing diverse pulse variations.
As noted in the main text, the exponential cutoff suppresses high-frequency noise and respects bandwidth and slew-rate limits of real ICF lasers.

Finally, we transform back via
\begin{equation}
\tilde{p}(t)=\max(\mathrm{IRFFT}(\tilde{P}), 0)
\end{equation}
to enforce continuity and non-negativity.

\section{Physics-Informed Inequality Loss: Full Derivations}
\label{sec:appendix:physical_loss}
This section provides complete derivations of the physics-informed inequality constraints used to regularize the surrogate during pretraining (Section~\ref{sec:method:physics_loss}).

\subsection{ICF Implosion in Three Stages}
\label{sec:appendix:icf_stages}

An ICF implosion proceeds through three stages with different dominant mechanisms. These stages motivate the inequality constraints below.

\paragraph{Stage~1: Laser ablation and rocket effect.}
An intense laser pulse irradiates the fuel capsule surface.
The outer material heats and expands outward as a plasma corona; by momentum conservation, this ablation drives an inward \emph{rocket reaction} that accelerates the remaining shell.
The mass ablation rate $\dot{m}_\alpha$ is set by laser intensity $I_L$ and wavelength $\lambda_L$ through Eq.~\eqref{eq:mdot_alpha_appendix}.
Because ablation removes mass from the shell, enclosed mass decreases monotonically during this stage.
The ablation pressure also imparts radial momentum, analogous to rocket thrust from mass ejection.

\paragraph{Stage~2: Implosion and compression.}
After the main drive ends, the inward-moving shell coasts and converges toward the center, compressing the fuel to extreme densities ($\bar\rho \sim 10^2$--$10^3\;\mathrm{g/cm^3}$). During this inertial coasting phase, no further momentum is injected; the shell kinetic energy is converted into internal energy as the fuel stagnates. The implosion velocity $v_c$ and remaining shell mass $m$ together determine the kinetic energy budget available for compression. Conservation laws demand that the enclosed momentum and kinetic energy cannot exceed the total impulse and energy delivered during Stage~1---these are the \emph{momentum} and \emph{energy confinement} constraints derived below.

\paragraph{Stage~3: Stagnation and thermonuclear burn.}
When the imploding shell decelerates against the compressed fuel core, kinetic energy is thermalized and the fuel reaches peak temperature and density.
If the Lawson criterion $\rho R T \geq 1.5\;\mathrm{g\cdot keV\cdot cm^{-2}}$ is met, a thermonuclear burn wave propagates before hydrodynamic disassembly quenches it.
Fusion yield depends on how efficiently implosion kinetic energy is converted into thermal energy at stagnation, which is set by the optimized pulse shape.

This three-stage picture explains why conservation constraints matter.
A surrogate that predicts too much enclosed mass or momentum is physically impossible.
The inequalities below encode mass, momentum, and energy bookkeeping constraints.

\subsection{Setup and Notation}

We consider a 1D Lagrangian capsule discretized into $N$ mass layers ($N=160$ in MULTI), indexed by $i\in\{1,\dots,N\}$, with fixed initial layer masses $m_i$.
Let $t_n$ denote the uniformly sampled time grid ($n=1,\dots,T$, $T=160$) and $\Delta t$ the timestep.
The total initial mass is
\begin{equation}
M_0=\sum_{i=1}^{N} m_i.
\end{equation}
After de-normalization, the surrogate predicts:
(i) boundary shock index $\hat I(t_n)$,
(ii) boundary radius $\hat R_I(t_n)$,
(iii) mass-weighted average velocity $\hat v_c(t_n)$,
together with other state histories.

\subsection{Differentiable Mass Operator}
Define the cumulative initial mass
\begin{equation}
M(k)=\sum_{i=1}^{k} m_i,\quad k\in\{0,1,\dots,N\},\, M(0)=0.
\end{equation}
We define a differentiable \emph{mass-from-index} operator $\mathcal{M}(\cdot)$ via linear interpolation:
\begin{equation}
\hat m(t_n) =\mathcal{M}(\hat I(t_n)) =(1-\delta_n)M(\lfloor \hat I_n\rfloor)+\delta_n M(\lceil \hat I_n\rceil)
\label{eq:mass_interp_appendix}
\end{equation}
where $\delta_n=\hat I_n-\lfloor \hat I_n\rfloor$, and $\hat I_n=\mathrm{clip}(\hat I(t_n),0,N)$.

\subsection{Steady Ablation Model}

Following the ablation rate formula \cite{Atzeni2004PhysicsInertialFusion}:
\begin{equation}
\dot m_\alpha(t)
=1.3\times 10^{6}
\Big(\frac{A}{2Z}\Big)^{2/3}
\Big(\frac{\lambda_L}{0.35\,\mu\mathrm{m}}\Big)^{-4/3}
\Big(\frac{I_L(t)}{10^{15}\,\mathrm{W/cm^2}}\Big)^{1/3}
\label{eq:mdot_alpha_appendix}
\end{equation}
When only laser power $P(t)$ is available, we use the boundary radius to approximate intensity:
\begin{equation}
I_L(t)\approx \frac{P(t)}{4\pi \hat R_I^2(t)}.
\label{eq:intensity_approx_appendix}
\end{equation}
The total ablated mass rate is then
\begin{equation}
\dot M_{\text{abl}}(t)=4\pi \hat R_I^2(t)\,\dot m_\alpha(t).
\label{eq:Mdot_appendix}
\end{equation}
In discrete time, the cumulative ablated mass and the remaining-mass \emph{upper bound} are
\begin{equation}
\Delta M_{\text{abl}}(t_n)=\dot M_{\text{abl}}(t_n)\Delta t,\,
M_{\lim}(t_n)=\max\Big(0,\,M_0-\sum_{k=1}^{n}\Delta M_{\text{abl}}(t_k)\Big).
\label{eq:mass_limit_appendix}
\end{equation}

\subsection{Physical Confinement Constraints}

\subsubsection{Momentum confinement}
Define the predicted enclosed momentum as
\begin{equation}
    \hat p(t_n)=\hat m(t_n)\,|\hat v_c(t_n)|.
    \label{eq:p_hat_appendix}
\end{equation}
We approximate the exhaust speed:
\begin{equation}
    u_{\text{ex}}(t)\approx 2c_s(t)\approx 2\sqrt{\frac{k_B T_I^*(t)}{m_{\text{ion}}}},
\label{eq:uex_appendix}
\end{equation}
The impulse budget up to $t_n$ is
\begin{equation}
    p_{\lim}(t_n)=\sum_{k=1}^{n}\Delta M_{\text{abl}}(t_k)\,u_{\text{ex}}(t_k).
    \label{eq:p_limit_appendix}
\end{equation}
The momentum constraint is enforced by a squared hinge penalty:
\begin{equation}
    \mathcal{L}_p
    =\frac{1}{T}\sum_{n=1}^{T}
    \Big[\max\big(0,\,\hat p(t_n)-p_{\lim}(t_n)\big)\Big]^2.
    \label{eq:L_p_appendix}
\end{equation}

\subsubsection{Energy confinement}
Define the predicted kinetic energy of the enclosed mass:
\begin{equation}
\hat E_{\text{kin}}(t_n)=\frac{1}{2}\hat m(t_n)\hat v_c^2(t_n).
\label{eq:E_hat_appendix}
\end{equation}
Using the remaining-mass bound $M_{\lim}(t_n)$:
\begin{equation}
u_{\text{imp}}(t_n)
=\sum_{k=1}^{n} u_{\text{ex}}(t_k)\,
\frac{\Delta M_{\text{abl}}(t_k)}{M_{\lim}(t_k)+\varepsilon},
\qquad \varepsilon>0.
\label{eq:uimp_appendix}
\end{equation}
The corresponding energy upper bound is
\begin{equation}
E_{\lim}(t_n)=\frac{1}{2}M_{\lim}(t_n)\,u_{\text{imp}}^2(t_n).
\label{eq:E_limit_appendix}
\end{equation}
\begin{equation}
\mathcal{L}_E
=\frac{1}{T}\sum_{n=1}^{T}
\Big[\max\big(0,\,\hat E_{\text{kin}}(t_n)-E_{\lim}(t_n)\big)\Big]^2.
\label{eq:L_E_appendix}
\end{equation}

\subsubsection{Mass confinement}
\begin{equation}
\mathcal{L}_m
=\frac{1}{T}\sum_{n=1}^{T}
\Big[\max\big(0,\,\hat m(t_n)-M_{\lim}(t_n)\big)\Big]^2.
\label{eq:L_m_appendix}
\end{equation}

\subsection{Overall Objective}
\begin{equation}
\mathcal{L}
=\mathcal{L}_{\text{data}}
+\lambda_m \mathcal{L}_m
+\lambda_p \mathcal{L}_p
+\lambda_E \mathcal{L}_E.
\label{eq:total_loss_appendix}
\end{equation}
In practice, these constraints matter most during ablation and compression.
When the boundary index collapses to $\hat I(t)\approx 0$ after ignition, $\hat m(t)$ becomes small and the penalties naturally deactivate.
Because the constraints operate on full trajectories, they apply only to the non-autoregressive surrogate.
Adding them to an autoregressive rollout would conflate physics violations with rollout error accumulation.

\section{PPO Objective Derivation for the One-Step Setting}
\label{sec:appendix:objective}

We derive the PPO-style surrogate objective used in Section~\ref{sec:method:ppo} for the one-step episodic setting.
The policy $\pi_\theta$ samples a latent action $\mathbf{z}\in\mathbb{R}^{d_z}$, which is decoded into a pulse
$\mathbf{p}=\mathbf{D}_\psi(\mathbf{z})$ by a fixed decoder $\mathbf{D}_\psi$.
The scalar reward is $r(\mathbf{p})$, hence the expected return is
\begin{equation}
\eta(\theta) \triangleq \mathbb{E}_{\mathbf{z}\sim \pi_\theta}\!\left[r\!\left(\mathbf{D}_\psi(\mathbf{z})\right)\right].
\end{equation}

\subsection{Bandit Advantage and Policy Difference Identity}
Given a reference (behavior) policy $\pi_{\text{old}}$, we define the one-step advantage as
\begin{equation}
A_{\text{old}}(\mathbf{z})
\triangleq
r\!\left(\mathbf{D}_\psi(\mathbf{z})\right)
-
\mathbb{E}_{\mathbf{z}'\sim \pi_{\text{old}}}\!\left[r\!\left(\mathbf{D}_\psi(\mathbf{z}')\right)\right].
\label{eq:appendix:adv_def}
\end{equation}

\begin{lemma}[Policy difference identity in the one-step setting]
\label{lem:appendix:bandit_identity}
For any policy $\pi_\theta$ and reference policy $\pi_{\text{old}}$,
\begin{equation}
\eta(\theta) = \eta(\pi_{\text{old}}) + \mathbb{E}_{\mathbf{z}\sim \pi_\theta}\!\left[A_{\text{old}}(\mathbf{z})\right].
\label{eq:appendix:bandit_identity}
\end{equation}
\end{lemma}

\begin{proof}
By definition,
\begin{align}
\eta(\theta) - \eta(\pi_{\text{old}})
&=
\mathbb{E}_{\mathbf{z}\sim \pi_\theta}\!\left[r\!\left(\mathbf{D}_\psi(\mathbf{z})\right)\right]
-
\mathbb{E}_{\mathbf{z}'\sim \pi_{\text{old}}}\!\left[r\!\left(\mathbf{D}_\psi(\mathbf{z}')\right)\right] \nonumber\\
&=
\mathbb{E}_{\mathbf{z}\sim \pi_\theta}\!\left[
r\!\left(\mathbf{D}_\psi(\mathbf{z})\right)
-
\mathbb{E}_{\mathbf{z}'\sim \pi_{\text{old}}}\!\left[r\!\left(\mathbf{D}_\psi(\mathbf{z}')\right)\right]
\right] \nonumber\\
&= \mathbb{E}_{\mathbf{z}\sim \pi_\theta}\!\left[A_{\text{old}}(\mathbf{z})\right],
\end{align}
where the second line uses that the $\pi_{\text{old}}$-expectation is a constant independent of $\mathbf{z}$.
\end{proof}

\subsection{Importance-Sampling Surrogate Objective}
Using importance sampling,
\begin{equation}
\mathbb{E}_{\mathbf{z}\sim \pi_\theta}[f(\mathbf{z})]
=
\mathbb{E}_{\mathbf{z}\sim \pi_{\text{old}}}\!\left[
\frac{\pi_\theta(\mathbf{z})}{\pi_{\text{old}}(\mathbf{z})} f(\mathbf{z})
\right].
\label{eq:appendix:is}
\end{equation}
Applying~\eqref{eq:appendix:is} to~\eqref{eq:appendix:bandit_identity} yields
\begin{equation}
\eta(\theta)
=
\eta(\pi_{\text{old}})
+
\mathbb{E}_{\mathbf{z}\sim \pi_{\text{old}}}\!\left[
\rho_\theta(\mathbf{z})\,A_{\text{old}}(\mathbf{z})
\right],
\qquad
\rho_\theta(\mathbf{z})\triangleq \frac{\pi_\theta(\mathbf{z})}{\pi_{\text{old}}(\mathbf{z})}.
\label{eq:appendix:surrogate_obj}
\end{equation}
Eq.~\eqref{eq:appendix:surrogate_obj} motivates the PPO-style surrogate objective used in the main text.

\end{document}